\documentclass{svjour3}
\usepackage{multirow}
\usepackage[utf8]{inputenc}
\usepackage[T1]{fontenc}
\usepackage{lmodern} 
\usepackage[bitstream-charter]{mathdesign}
\usepackage{inconsolata}
\usepackage{pgfplots}
\usepackage{tikz}
\usepackage{subcaption}
\usepackage{listings}
\usepackage{amsmath}
\usepackage{hyperref}
\usepackage{booktabs}

\smartqed

\usepackage{algorithm}
\usepackage{algpseudocode}
\usepackage{cite}
\usepackage{placeins}

\usetikzlibrary{plotmarks}

\let\oldendproof\endproof
\renewcommand{\endproof}{\qed\oldendproof}

\newenvironment{sketchproof}{%
  \proof}{\endproof}

\pgfplotsset{compat=1.10}

\journalname{Machine learning}
\title{Logical reduction of metarules}
\author{Andrew Cropper \and Sophie Tourret}
\begin{document}

\institute{A. Cropper\at
              University of Oxford, UK\\
              \email{andrew.cropper@cs.ox.ac.uk}
            \and
                Sophie Tourret \at
                Max Planck Institute for Informatics, Germany\\
              \email{stourret@mpi-inf.mpg.de}
}
\newcommand{\nsym}{\ensuremath{\not\sqsubset}}
\newcommand{\sym}{\ensuremath{\sqsubset}}
\newcommand{\y}{\ensuremath{\checkmark}}
\newcommand{\n}{\ensuremath{\times}}
\newcommand{\M}[2]{\ensuremath{\mathcal{M}^{#1}_{#2}}}

\newcommand{\C}[2]{\ensuremath{\mathcal{C}^{#1}_{#2}}}
\newcommand{\D}[2]{\ensuremath{\mathcal{D}^{#1}_{#2}}}
\newcommand{\K}[2]{\ensuremath{\mathcal{K}^{#1}_{#2}}}
\newcommand{\U}[2]{\ensuremath{\mathcal{U}^{#1}_{#2}}}

\newcommand{\todo}[1]{\textbf{TODO: #1}}
\newcommand{\tw}[1]{\texttt{#1}}

\newtheorem{innercustomthm}{Theorem}
\newenvironment{customthm}[1]
  {\renewcommand\theinnercustomthm{#1}\innercustomthm}
  {\endinnercustomthm}

\newtheorem{innercustomprp}{Proposition}
\newenvironment{customprp}[1]
  {\renewcommand\theinnercustomprp{#1}\innercustomprp}
  {\endinnercustomprp}

\newcommand{\perm}[1]{\ensuremath{\mathcal{P}(#1)}}



\lstnewenvironment{myalgorithm}[1][] 
{
    \lstset{ 
        mathescape=true,
        numbers=left,
        escapeinside={*}{*},
        keywordstyle=\color{black}\bfseries\em,
        keywords={,input, output, return, datatype, function, func, in, if, else, foreach, while, begin, end, } 
        numbers=left,
        xleftmargin=.04\textwidth,
        #1 
    }
}
{}

\maketitle

\begin{abstract}
Many forms of inductive logic programming (ILP) use \emph{metarules}, second-order Horn clauses, to define the structure of learnable programs and thus the hypothesis space.
Deciding which metarules to use for a given learning task is a major open problem and is a trade-off between efficiency and expressivity: the hypothesis space grows given more metarules, so we wish to use fewer metarules, but if we use too few metarules then we lose expressivity.
In this paper, we study whether fragments of metarules can be logically reduced to minimal finite subsets.
We consider two traditional forms of logical reduction: subsumption and entailment.
We also consider a new reduction technique called \emph{derivation reduction}, which is based on SLD-resolution.
We compute reduced sets of metarules for fragments relevant to ILP and theoretically show whether these reduced sets are reductions for more general infinite fragments.
We experimentally compare learning with reduced sets of metarules on three domains: Michalski trains, string transformations, and game rules.
In general, derivation reduced sets of metarules outperforms subsumption and entailment reduced sets, both in terms of predictive accuracies and learning times.
\end{abstract}
\section{Introduction}
\label{sec:intro}

Many forms of inductive logic programming (ILP) \cite{mugg:metagold,crop:metaopt,emde:metarules,raedt:clint,dialogs,mobal,hexmil,andreas,albarghouthi2017constraint,ALPS,wang2014structure,evans:dilp} use second-order Horn clauses, called \emph{metarules}\footnote{Metarules are also called \emph{program schemata} \cite{dialogs}, \emph{second-order schemata} \cite{raedt:clint}, and \emph{clause templates} \cite{albarghouthi2017constraint}, amongst many other names.} as a form of declarative bias \cite{raedt:decbias}.
Metarules define the structure of learnable programs which in turn defines the hypothesis space. For instance, to learn the \emph{grandparent/2} relation given the \emph{parent/2} relation, the \emph{chain} metarule would be suitable:

\begin{center}
  \begin{tabular}{l}
    $P(A,B) \leftarrow Q(A,C), R(C,B)$\\
  \end{tabular}
\end{center}

\noindent
In this metarule\footnote{The fully quantified rule is $\exists P \exists Q \exists R \forall A \forall B \forall C \; P(A,B) \leftarrow Q(A,C), R(C,B)$.} the letters $P$, $Q$, and $R$ denote existentially quantified second-order variables (variables that can be bound to predicate symbols) and the letters $A$, $B$ and $C$ denote universally quantified first-order variables (variables that can be bound to constant symbols). Given the \emph{chain} metarule, the background \emph{parent/2} relation, and examples of the \emph{grandparent/2} relation, ILP approaches will try to find suitable substitutions for the existentially quantified second-order variables, such as the substitutions \{P/grandparent, Q/parent, R/parent\} to induce the theory:

\begin{center}
  \begin{tabular}{l}
    $\mathit{grandparent}(A,B) \leftarrow \mathit{parent}(A,C), \mathit{parent}(C,B)$\\
  \end{tabular}
\end{center}

\noindent


\noindent
However, despite the widespread use of metarules, there is little work determining which metarules to use for a given learning task.
Instead, suitable metarules are assumed to be given as part of the background knowledge, and are often used without any theoretical justification.
Deciding which metarules to use for a given learning task is a major open challenge \cite{crop:thesis,crop:minmeta} and is a trade-off between efficiency and expressivity: the hypothesis space grows given more metarules \cite{mugg:metabias,crop:minmeta}, so we wish to use fewer metarules, but if we use too few metarules then we lose expressivity. For instance, it is impossible to learn the \emph{grandparent/2} relation using only metarules with monadic predicates.

In this paper, we study whether potentially infinite fragments of metarules can be logically reduced to minimal, or irreducible, finite subsets, where a fragment is a syntactically restricted subset of a logical theory \cite{bradley2007calculus}.

 Cropper and Muggleton \cite{crop:minmeta} first studied this problem. They used Progol's entailment reduction algorithm \cite{mugg:progol} to identify entailment reduced sets of metarules, where a clause $C$ is entailment redundant in a clausal theory $T \cup \{C\}$ when $T \models C$.
 To illustrate entailment redundancy, consider the following first-order clausal theory $T_1$, where $p$, $q$, $r$, and $s$ are first-order predicates:

\begin{center}
  \begin{tabular}{l}
    $C_1 = p(A,B) \leftarrow q(A,B)$\\
    $C_2 = p(A,B) \leftarrow q(A,B),r(A)$\\
    $C_3 = p(A,B) \leftarrow q(A,B),r(A),s(B,C)$
  \end{tabular}
\end{center}

\noindent
In $T_1$ the clauses $C_2$ and $C_3$ are entailment redundant because they are both logical consequences of $C_1$, i.e. $\{C_1\} \models \{C_2,C_3\}$. Because $\{C_1\}$ cannot be reduced, it is a minimal entailment reduction of $T_1$.

Cropper and Muggleton showed that in some cases as few as two metarules are sufficient to entail an infinite fragment of \emph{chained}\footnote{A chained dyadic Datalog clause has the restriction that every first-order variable in a clause appears in exactly two literals and a path connects every literal in the body of C to the head of C. In other words, a chained dyadic Datalog clause has the form $P_0(X_0,X_1) \leftarrow P_1(X_0,X_2), P_2(X_2,X_3), \dots, P_n(X_n,X_1)$ where the order of the arguments in the literals does not matter.} second-order dyadic Datalog \cite{crop:minmeta}.
They also showed that learning with minimal sets of metarules improves predictive accuracies and reduces learning times compared to non-minimal sets.
To illustrate how a finite subset of metarules could entail an infinite set, consider the set of metarules with only monadic literals and a single first-order variable $A$:

\begin{center}
  \begin{tabular}{l}
    $M_1 = P(A) \leftarrow T_1(A)$\\
    $M_2 = P(A) \leftarrow T_1(A),T_2(A)$\\
    $M_3 = P(A) \leftarrow T_1(A),T_2(A),T_3(A)$\\
    $\dots$\\
    $M_{n} = P(A) \leftarrow T_1(A),T_2(A),\dots,T_{n}(A)$\\
    $\dots$\\
  \end{tabular}
\end{center}

\noindent
Although this set is infinite it can be entailment reduced to the single metarule $M_1$ because it implies the rest of the theory.

However, in this paper, we claim that entailment reduction is not always the most appropriate form of reduction.
For instance, suppose you want to learn the \emph{father/2} relation given the background relations \emph{parent/2}, \emph{male/1}, and \emph{female/1}. Then a suitable hypothesis is:

\begin{center}
  \begin{tabular}{l}
    $\mathit{father}(A,B) \leftarrow \mathit{parent}(A,B), \mathit{male}(A)$\\
  \end{tabular}
\end{center}

\noindent
To learn such a hypothesis one would need a metarule of the form $P(A,B) \leftarrow Q(A,B),R(A)$.
Now suppose you have the metarules:

\begin{center}
  \begin{tabular}{l}
    $M_1 = P(A,B) \leftarrow Q(A,B)$\\
    $M_2 = P(A,B) \leftarrow Q(A,B),R(A)$
  \end{tabular}
\end{center}

\noindent
Running entailment reduction on these metarules would remove $M_2$ because it is a logical consequence of $M_1$.
However, it is impossible to learn the intended \emph{father/2} relation given only $M_1$.
As this example shows, entailment reduction can be too strong because it can remove metarules necessary to specialise a clause, where $M_2$ can be seen as a specialisation of $M_1$.

To address this issue, we introduce \emph{derivation reduction}, a new form of reduction based on derivations, which we claim is a more suitable form of reduction for reducing sets of metarules.
Let $\vdash$ represent derivability in SLD-resolution\footnote{We use $\vdash$ to represent derivability of both first-order and second-order clauses. In practice we reason about second-order clauses using first-order resolution via encapsulation \cite{crop:minmeta}, which we describe in Section \ref{sec:encapsulation}.} \cite{sld-resolution}, then a Horn clause $C$ is derivationally redundant in a Horn theory $T \cup \{C\}$ when $T \vdash C$.
A Horn theory is derivationally irreducible if it contains no derivationally redundant clauses.
To illustrate the difference between entailment and derivation reduction, consider the metarules:

\begin{center}
  \begin{tabular}{l}
    $M_1 = P(A,B) \leftarrow Q(A,B)$\\
    $M_2 = P(A,B) \leftarrow Q(A,B),R(A)$\\
    $M_3 = P(A,B) \leftarrow Q(A,B),R(A,B)$\\
    $M_4 = P(A,B) \leftarrow Q(A,B),R(A,B),S(A,B)$
  \end{tabular}
\end{center}

\noindent
Running entailment reduction on these metarules would result in the reduction $\{M_1\}$ because $M_1$ entails the rest of the theory.
Likewise, running \emph{subsumption reduction} \cite{plotkin:thesis} (described in detail in Section \ref{sec:subsumption-reduction}) would also result in the reduction $\{M_1\}$.
By contrast, running derivation reduction would only remove $M_4$ because it can be derived by self-resolving $M_3$.
The remaining metarules $M_2$ and $M_3$ are not derivationally redundant because there is no way to derive them from the other metarules.

\subsection{Contributions}
In the rest of this paper, we study whether fragments of metarules relevant to ILP can be logically reduced to minimal finite subsets. We study three forms of reduction: subsumption \cite{robinson:resolution}, entailment \cite{mugg:progol}, and derivation. We also study how learning with reduced sets of metarules affects learning performance. To do so, we supply Metagol \cite{metagol}, a meta-interpretive learning (MIL) \cite{mugg:metalearn,mugg:metagold,crop:metafunc} implementation, with different reduced sets of metarules and measure the resulting learning performance on three domains: Michalski trains \cite{michalski:trains}, string transformations, and game rules \cite{iggp}. In general, using derivation reduced sets of metarules outperforms using subsumption and entailment reduced sets, both in terms of predictive accuracies and learning times. Overall, our specific contributions are:

\begin{itemize}
\item We describe the logical reduction problem (Section 3).
\item We describe subsumption and entailment reduction, and introduce derivation reduction, the problem of removing derivationally redundant clauses from a clausal theory (Section \ref{sec:framework}).
\item We study the decidability of the three reduction problems and show, for instance, that the derivation reduction problem is undecidable for arbitrary Horn theories (Section \ref{sec:framework}).
\item We introduce two general reduction algorithms that take a reduction relation as a parameter.
We also study their complexity (Section \ref{sec:algorithms}).
\item We run the reduction algorithms on finite sets of metarules to identify minimal sets (Section \ref{sec:results}).
\item We theoretically show whether infinite fragments of metarules can be logically reduced to finite sets (Section \ref{sec:results}).
\item We experimentally compare the learning performance of Metagol when supplied with reduced sets of metarules on three domains: Michalski trains, string transformations, and game rules (Section \ref{sec:experiments}).
\end{itemize}
\section{Related work}
\label{sec:related}

This section describes work related to this paper, mostly work on logical reduction techniques. We first, however, describe work related to MIL and metarules.



\subsection{Meta-interpretive learning}

Although the study of metarules has implications for many ILP approaches \cite{mugg:metagold,crop:metaopt,emde:metarules,raedt:clint,dialogs,mobal,hexmil,andreas,albarghouthi2017constraint,wang2014structure,evans:dilp,ALPS}, we focus on meta-interpretive learning (MIL), a form of ILP based on a Prolog meta-interpreter\footnote{Although the MIL problem has also been encoded as an ASP problem \cite{hexmil}.}.
The key difference between a MIL learner and a standard Prolog meta-interpreter is that whereas a standard Prolog meta-interpreter attempts to prove a goal by repeatedly fetching first-order clauses whose heads unify with a given goal, a MIL learner additionally attempts to prove a goal by fetching second-order metarules, supplied as background knowledge (BK), whose heads unify with the goal.
The resulting meta-substitutions are saved and can be reused in later proofs.
Following the proof of a set of goals, a logic program is formed by projecting the meta-substitutions onto their corresponding metarules, allowing for a form of ILP which supports predicate invention and learning recursive theories.

Most existing work on MIL has assumed suitable metarules as input to the problem, or has used metarules without any theoretical justification.
In this paper, we try to address this issue by identifying minimal sets of metarules for interesting fragments of logic, such as Datalog, from which a MIL system can theoretically learn any logic program.

\subsection{Metarules}

McCarthy \cite{DBLP:conf/mi/McCarthy95} and Lloyd \cite{lloyd:logiclearning} advocated using second-order logic to represent knowledge. Similarly, Muggleton et al. \cite{mlj:ilp20} argued that using second-order representations in ILP provides more flexible ways of representing BK compared to existing methods. Metarules are second-order Horn clauses and are used as a form of declarative bias \cite{nedellec1996declarative,raedt:decbias} to determine the structure of learnable programs which in turn defines the hypothesis space. In contrast to other forms of declarative bias, such as modes \cite{mugg:progol} or grammars \cite{cohen:grammarbias}, metarules are logical statements that can be reasoned about, such as to reason about the redundancy of sets of metarules, which we explore in this paper.

Metarules were introduced in the Blip system \cite{emde:metarules}.
Kietz and Wrobel \cite{mobal} studied generality measures for metarules in the RDT system.
A generality order is necessary because the RDT system searches the hypothesis space (which is defined by the metarules) in a top-down general-to-specific order.
A key difference between RDT and MIL is that whereas RDT requires metarules of increasing complexity (e.g. rules with an increasing number of literals in the body), MIL derives more complex metarules through SLD-resolution. This point is important because this ability allows MIL to start from smaller sets of primitive metarules.
In this paper we try to identify such primitive sets.

Using metarules to build a logic program is similar to the use of refinement operators in ILP \cite{shapiro:thesis,ilp:book} to build a definite clause literal-by-literal\footnote{MIL uses example driven test-incorporation for finding consistent programs as opposed to the generate-and-test approach of clause refinement.}. As with refinement operators, it seems reasonable to ask about completeness and irredundancy of a set of metarules, which we explore in this paper.


\subsection{Logical redundancy}
Detecting and eliminating redundancy in a clausal theory is useful in many areas of computer science. In ILP logically reducing a theory is useful to remove redundancy from a hypothesis space to improve learning performance \cite{Fonseca:ilp04,crop:minmeta}. In general, simplifying or reducing a theory often makes a theory easier to understand and use, and may also have computational efficiency advantages.

\subsubsection{Literal redundancy}
Plotkin \cite{plotkin:thesis} used subsumption to decide whether a literal is redundant in a first-order clause. Joyner \cite{DBLP:journals/jacm/Joyner76} independently investigated the same problem, which he called \emph{clause condensation}, where a condensation of a clause $C$ is a minimum cardinality subset $C'$ of $C$ such that $C' \models C$. Gottlob and Ferm\"uller \cite{gottlob1993removing} improved Joyner's algorithm and also showed that determining whether a clause is condensed is coNP-complete. In contrast to removing redundant literals, we focus on removing redundant clauses.

\subsubsection{Clause redundancy}
Plotkin \cite{plotkin:thesis} introduced methods to decide whether a clause is subsumption redundant in a first-order clausal theory.
This problem has also been extensively studied in the context of first-order logic with equality due to its application in superposition-based theorem proving \cite{DBLP:conf/birthday/HillenbrandPWW13,DBLP:journals/aicom/WeidenbachW10}.
The same problem, and slight variants, has been extensively studied in the propositional case \cite{liberatore1,liberatore2}. Removing redundant clauses has numerous applications, such as to improve the efficiency of SAT \cite{heule2015clause}. In contrast to these works, we focus on reducing theories formed of second-order Horn clauses (without equality), which to our knowledge has not yet been extensively explored. Another difference is that we additionally study redundancy based on SLD-derivations.


Cropper and Muggleton \cite{crop:minmeta} used Progol's entailment-reduction algorithm \cite{mugg:progol} to identify irreducible, or minimal, sets of metarules.
Their approach removed entailment redundant clauses from sets of metarules.
They identified theories that are (1) entailment complete for certain fragments of second-order Horn logic, and (2) minimal or irreducible in that no further reductions are possible.
They demonstrated that in some cases as few as two clauses are sufficient to entail an infinite theory.
However, they only considered small and highly constrained fragments of metarules.
In particular, they focused on an \emph{exactly-two-connected} fragment of metarules where each literal is dyadic and each first-order variable appears exactly twice in distinct literals.
However, as discussed in the introduction, entailment reduction is not always the most appropriate form of reduction because it can remove metarules necessary to specialise a clause.
Therefore, in this paper, we go beyond entailment reduction and introduce derivation reduction.
We also consider more general fragments of metarules, such as a fragment of metarules sufficient to learn Datalog programs.

Cropper and Tourret \cite{crop:dreduce} introduced the derivation reduction problem and studied whether sets of metarules could be derivationally reduced.
They considered the \emph{exactly-two-connected} fragment previously considered by Cropper and Muggleton and a \emph{two-connected} fragment in which every variable appears at least twice, which is analogous to our singleton-free fragment (Section \ref{sec:singletonfree}).
They used graph theoretic methods to show that certain fragments could not be completely derivationally reduced.
They demonstrated on the Michalski trains dataset that the partially derivationally reduced set of metarules outperforms the entailment reduced set.
In similar work Cropper and Tourret elaborated on their graph theoretic techniques and expanded the results to unconstrained resolution \cite{crop:sldres}.

In this paper, we go beyond the work of Cropper and Tourret in several ways.
First, we consider more general fragments of metarules, including \emph{connected} and \emph{Datalog} fragments.
We additionally consider fragments with zero arity literals.
In all cases we provide additional theoretical results showing whether certain fragments can be reduced, and, where possible, show the actual reductions.
Second, Cropper and Tourret \cite{crop:sldres} focused on derivation reduction modulo first-order variable unification, i.e. they considered the case where factorisation \cite{ilp:book} was allowed when resolving two clauses, which is not implemented in practice in current MIL systems.
For this reason, although Section 5 in \cite{crop:sldres} and Section \ref{sec:connected} in the present paper seemingly consider the same problem, the results are opposite to one another.
Third, in addition to entailment and derivation reduction, we also consider subsumption reduction.
We provide more theoretical results on the decidability of the reduction problems, such as showing a decidable case for derivation reduction (Theorem \ref{thm:decidable}).
Fourth, we describe the reduction algorithms and discuss their computational complexity.
Finally, we corroborate the experimental results of Cropper and Tourret on Michalski's train problem \cite{crop:dreduce} and provide additional experimental results on two more domains: real-world string transformations and inducing Datalog game rules from observations.

\subsubsection{Theory minimisation}
We focus on removing clauses from a clausal theory. A related yet distinct topic is theory minimisation where the goal is to find a minimum equivalent formula to a given input formula. This topic is often studied in propositional logic \cite{DBLP:conf/ijcai/HemaspaandraS11}. The minimisation problem allows for the introduction of new clauses. By contrast, the reduction problem studied in this paper does not allow for the introduction of new clauses and instead only allows for the removal of redundant clauses.

\subsubsection{Prime implicates}
Implicates of a theory $T$ are the clauses that are entailed by $T$ and are called prime when they do not themselves entail other implicates of $T$.
This notion differs from the subsumption and derivation reduction because it focuses on entailment, and it differs from entailment reduction 
because (1) the notion of a prime implicate has been studied only in propositional, first-order, and some modal logics \cite{marquis2000consequence,echenim2015quantifierfree,bienvenu2007prime}; (2) the generation of prime implicates allows for the introduction of new clauses in the formula.

%
\section{Logical reduction}
\label{sec:framework}

We now introduce the reduction problem: the problem of finding redundant clauses in a theory. We first describe the reduction problem starting with preliminaries, and then describe three instances of the problem. The first two instances are based on existing logical reduction methods: subsumption and entailment. The third instance is a new form of reduction introduced in \cite{crop:dreduce} based on SLD-derivations.

\subsection{Preliminaries}

We assume familiarity with logic programming notation \cite{lloyd:book} but we restate some key terminology.
A {\em clause} is a disjunction of literals.
A {\em clausal theory} is a set of clauses.
A {\em Horn} clause is a clause with at most one positive literal.
A Horn theory is a set of Horn clauses.
A {\em definite} clause is a Horn clause with exactly one positive literal.
A Horn clause is a \emph{Datalog} clause if (1) it contains no function symbols, and (2) every variable that appears in the head of the clause also appears in a positive (i.e. not negated) literal in the body of the clause\footnote{Datalog also imposes additional constraints on negation in the body of a clause, but because we disallow negation in the body we omit these constraints for simplicity.}.
We denote the powerset of the set $S$ as $2^S$.

\subsubsection{Metarules}
\label{sec:metarules}

Although the reduction problem applies to any clausal theory, we focus on theories formed of metarules:

\begin{definition}[Metarule]
\label{def:metarule}
A metarule is a second-order Horn clause of the form:
$$A_0 \leftarrow A_1, \; \dots \;, \; \; A_m$$
where each $A_i$ is a literal of the form $P(T_1,\dots,T_n )$ where $P$ is either a predicate symbol or a second-order variable that can be substituted by a predicate symbol, and each $T_i$ is either a constant symbol or a first-order variable that can be substituted by a constant symbol.
\end{definition}

\begin{table}[ht]
\centering
\normalsize
\begin{tabular}{|l|l|}
\hline
\textbf{Name} & \textbf{Metarule}\\ \hline
Indent$_1$ & $P(A) \leftarrow Q(A)$\\
DIndent$_1$ & $P(A) \leftarrow Q(A),R(A)$\\
Indent$_2$ & $P(A,B) \leftarrow Q(A,B)$\\
DIndent$_2$ & $P(A,B) \leftarrow Q(A,B),R(A,B)$\\
Precon & $P(A,B) \leftarrow Q(A),R(A,B)$\\
Postcon & $P(A,B) \leftarrow Q(A,B),R(B)$\\
Curry & $P(A,B) \leftarrow Q(A,B,R)$\\
Chain & $P(A,B) \leftarrow Q(A,C), R(C,B)$\\
\hline
\end{tabular}
\caption{Example metarules. The letters $P$, $Q$, and $R$ denote existentially quantified second-order variables. The letters $A$, $B$, and $C$ denote universally quantified first-order variables.}
\label{tab:metarules}
\end{table}

\noindent
Table \ref{tab:metarules} shows a selection of metarules commonly used in the MIL literature \cite{crop:metafunc,crop:metaopt,crop:typed,crop:metagolo,crop:datacurate}.
As Definition \ref{def:metarule} states, metarules may include predicate and constant symbols.
However, we focus on the more general case where metarules only contain variables\footnote{By more general we mean we focus on metarules that are independent of any particular ILP problem with particular predicate and constant symbols.}.
In addition, although metarules can be any Horn clauses, we focus on definite clauses with at least one body literal, i.e. we disallow facts, because their inclusion leads to uninteresting reductions, where in almost all such cases the theories can be reduced to a single fact\footnote{
For instance, the metarule $P(A) \leftarrow$ entails and subsumes every metarule with a monadic head.
}
We denote the infinite set of all such metarules as \M{}{}.
We focus on {\em fragments} of \M{}{}, where a fragment is a syntactically restricted subset of a theory \cite{bradley2007calculus}:

\begin{definition}[The fragment \M{a}{m}]
\label{def:ham}
We denote as \M{a}{m} the fragment of \M{}{} where each literal has arity at most $a$ and each clause has at most $m$ literals in the body.
We replace $a$ by the explicit set of arities when we restrict the allowed arities further.
\end{definition}


\begin{example}
\M{\{2\}}{2} is a subset of \M{}{} where each predicate has arity 2 and each clause has at most 2 body literals.
\end{example}

\begin{example}
\M{\{2\}}{m} is a subset of \M{}{} where each predicate has arity 2 and each clause has at most $m$ body literals.
\end{example}

\begin{example}
\M{\{0,2\}}{m} is a subset of \M{}{} where each predicate has arity 0 or 2 and each clause has at most $m$ body literals.
\end{example}

\begin{example}
\M{a}{\{1,2\}} is a subset of \M{}{} where each predicate has arity at most $a$ and each clause has either 1 or 2 body literals.
\end{example}

\noindent
Let $T$ be a clausal theory. Then we say that $T$ is in the fragment \M{a}{m} if and only if each clause in $T$ is in \M{a}{m}.
\subsection{Meta-interpretive learning}
\label{sec:mil}

In Section \ref{sec:experiments} we conduct experiments to see whether using reduced sets of metarules can improve learning performance.
The primary purpose of the experiments is to test our claim that entailment reduction is not always the most appropriate form of reduction.
Our experiments focus on MIL.
For self-containment, we briefly describe MIL.

\begin{definition}[\textbf{MIL input}]
\label{def:milinput}
An MIL input is a tuple $(B,E^+,E^-,M)$ where:
\begin{itemize}
    \item $B$ is a set of Horn clauses denoting background knowledge
    \item $E^+$ and $E^-$ are disjoint sets of ground atoms representing positive and negative examples respectively
    \item $M$ is a set of metarules
\end{itemize}
\end{definition}

\noindent
The MIL problem is defined from a MIL input:

\begin{definition}[\textbf{MIL problem}]
\label{def:milproblem}
Given a MIL input $(B,E^+,E^-,M)$, the MIL problem is to return a logic program hypothesis $H$ such that:
\begin{itemize}
  \item $\forall c \in H, \exists m \in M$ such that $c=m\theta$, where $\theta$ is a substitution that grounds all the
  existentially quantified variables in $m$
  \item $H \cup B \models E^{+}$
  \item $H \cup B \not\models E^{-}$
\end{itemize}
We call $H$ a solution to the MIL problem.
\end{definition}







\noindent
The metarules and background define the hypothesis space.
To explain our experimental results in Section \ref{sec:experiments}, it is important to understand the effect that metarules have on the size of the MIL hypothesis space, and thus on learning performance.
The following result generalises previous results \cite{mugg:metabias,crop:metafunc}:

\begin{theorem}[MIL hypothesis space]
\label{thm:hypspace}
Given $p$ predicate symbols and $k$ metarules in \M{a}{m}, the number of programs expressible with $n$ clauses is at most $(p^{m+1}k)^n$.
\end{theorem}
\begin{proof}
The number of first-order clauses which can be constructed from a \M{a}{m} metarule given $p$ predicate symbols is at most $p^{m+1}$ because for a given metarule there are at most $m+1$ predicate variables with at most $p^{m+1}$ possible substitutions.
Therefore the set of such clauses $S$ which can be formed from $k$ distinct metarules in \M{a}{m} using $p$ predicate symbols has cardinality at most $p^{m+1}k$.
It follows that the number of programs which can be formed from a selection of $n$ clauses chosen from $S$ is at most $(p^{m+1}k)^n$.
\end{proof}


\noindent
Theorem \ref{thm:hypspace} shows that the MIL hypothesis space increases given more metarules.
The Blumer bound \cite{blumer:bound}\footnote{The Blumer bound is a reformulation of Lemma 2.1 in \cite{blumer:bound}.}, says that given two hypothesis spaces, searching the smaller space will result in fewer errors compared to the larger space, assuming that the target hypothesis is in both spaces. This result suggests that we should consider removing redundant metarules to improve the learning performance. We explore this idea in the rest of the paper.
\subsection{Encapsulation}
\label{sec:encapsulation}

To reason about metarules (especially when running the Prolog implementations of the reduction algorithms), we use a method called encapsulation \cite{crop:minmeta} to transform a second-order logic program to a first-order logic program.
We first define encapsulation for atoms:

\begin{definition}[Atomic encapsulation]
Let $A$ be a second-order or first-order atom of the form $P(T_{1},..,T_{n})$. Then $enc(A) = enc(P,T_{1},..,T_{n})$ is the encapsulation of $A$.
\end{definition}

\noindent
For instance, the encapsulation of the atom \emph{parent(ann,andy)} is \emph{enc(parent,ann,andy)}.
Note that encapsulation essentially ignores the quantification of variables in metarules by treating all variables, including predicate variables, as first-order universally quantified variables of the first-order $enc$ predicate.
In particular, replacing existential quantifiers with universal quantifiers on predicate variables is fine for our work because we only reason about the form of metarules, not their semantics, i.e. we treat metarules as templates for first-order clauses.
We extend atomic encapsulation to logic programs:

\begin{definition}[Program encapsulation]
The logic program $enc(P)$ is the encapsulation of the logic program $P$ in the case $enc(P)$ is formed by replacing all atoms $A$ in $P$ by $enc(A)$.
\end{definition}

\noindent
For example, the encapsulation of the metarule $P(A,B) \leftarrow Q(A,C), R(C,B)$ is $enc(P,A,B) \leftarrow enc(Q,A,C), enc(R,C,B)$.
We extend encapsulation to interpretations \cite{ilp:book} of logic programs:





\begin{definition}[Interpretation encapsulation]
Let $I$ be an interpretation over the predicate and constant symbols in a logic program.
Then the encapsulated interpretation $enc(I)$ is formed by replacing each atom $A$ in $I$ by $enc(A)$.
\end{definition}

\noindent
We now have the proposition:

\begin{proposition}[Encapsulation models \cite{crop:minmeta}]
\label{minmeta:prop1}
The second-order logic program $P$ has a model $M$ if and only if $enc(P)$ has the model $enc(M)$.
\end{proposition}
\begin{proof}
Follows trivially from the definitions of encapsulated programs and interpretations.
\end{proof}

\noindent
We can extend the definition of entailment to logic programs:

\begin{proposition}[Entailment \cite{crop:minmeta}]
Let $P$ and $Q$ be second-order logic programs. Then $P\models Q$ if and only if every model $enc(M)$ of $enc(P)$ is also a model of $enc(Q)$.
\end{proposition}
\begin{proof}
Follows immediately from Proposition \ref{minmeta:prop1}.
\end{proof}

\noindent
These results allow us to reason about metarules using standard first-order logic.
In the rest of the paper all the reasoning about second-order theories is performed at the first-order level.
However, to aid the readability we continue to write non-encapsulated metarules in the rest of the paper, i.e. we will continue to refer to sets of metarules as second-order theories.
\subsection{Logical reduction problem}

We now describe the logical reduction problem.
For the clarity of the paper, and to avoid repeating definitions for each form of reduction that we consider (entailment, subsumption, and derivability), we describe a general reduction problem which is parametrised by a binary relation \sym{} defined over any clausal theory, although in the case of derivability, \sym{} is in fact only defined over Horn clauses.
Our only constraint on the relation \sym{} is that if $A\sym{}B$, $A\subseteq A'$ and $B'\subseteq B$ then $A'\sym{}B'$.
We first define a redundant clause:

\begin{definition}[\sym{}-redundant clause]
\label{def:rclause}
The clause $C$ is \sym{}-redundant in the clausal theory $T \cup \{C\}$ whenever $T$ \sym{} $\{C\}$.
\end{definition}

\noindent
In a slight abuse of notation, we allow Definition \ref{def:rclause} to also refer to a single clause, i.e.\ in our notation $T$ \sym{} $C$ is the same as $T$ \sym{} $\{C\}$. We define a reduced theory:
\begin{definition}[\sym{}-reduced theory]
A clausal theory is \sym{}-reduced if and only if it is finite and it does not contain any \sym{}-redundant clauses.
\end{definition}

\noindent
We define the input to the reduction problem:
\begin{definition}[\sym{}-reduction input]
A reduction input is a pair $(T,$\sym{}$)$ where $T$ is a clausal theory and \sym{} is a binary relation over a clausal theory.
\end{definition}

\noindent
Note that a reduction input may (and often will) be an infinite clausal theory. We define the reduction problem:

\begin{definition}[\sym{}-reduction problem]
\label{def:sym-prob}
Let $(T,$\sym{}$)$ be a reduction input. Then the \sym{}-reduction problem is to find a finite theory $T' \subseteq  T$ such that (1) $T'$ \sym{} $T$ (i.e. $T'$ \sym{} $C$ for every clause $C$ in $T$), and (2) $T'$ is \sym{}-reduced. We call $T'$ a \emph{\sym{}-reduction}.
\end{definition}

\noindent
Although the input to a \sym{}-reduction problem may contain an infinite theory, the output (a \emph{\sym{}-reduction}) must be a finite theory. We also introduce a variant of the \sym{}-reduction problem where the reduction must obey certain syntactic restrictions:

\begin{definition}[\M{a}{m}-\sym{}-reduction problem]
\label{def:hamprob}
Let ($T$,\sym{},\M{a}{m}) be a triple, where the first two elements are as in a standard reduction input and \M{a}{m} is a target reduction theory. Then the \M{a}{m}-\sym{}-reduction problem is to find a finite theory $T' \subseteq  T$ such that (1) $T'$ is a \sym-reduction of $T$, and (2) $T'$ is in \M{a}{m}.
\end{definition}
\subsection{Subsumption reduction}
\label{sec:subsumption-reduction}

The first form of reduction we consider is based on subsumption, which, as discussed in Section \ref{sec:related}, is often used to eliminate redundancy in a clausal theory:

\begin{definition}[Subsumption]
A clause $C$ subsumes a clause $D$, denoted as $C \preceq D$, if there exists a substitution $\theta$ such that $C\theta \subseteq D$.
\end{definition}

\noindent
Note that if a clause $C$ subsumes a clause $D$ then $C \models D$ \cite{robinson:resolution}.
However, if $C \models D$ then it does not necessarily follow that $C \preceq D$.
Subsumption can therefore be seen as being weaker than entailment.
Whereas checking entailment between clauses is undecidable \cite{church:problem}, Robinson \cite{robinson:resolution} showed that checking subsumption between clauses is decidable (although in general deciding subsumption is a NP-complete problem \cite{ilp:book}).

If $T$ is a clausal theory then the pair $(T,\preceq)$ is an input to the \sym{}-reduction problem, which leads to the \emph{subsumption reduction} problem (S-reduction problem). We show that the S-reduction problem is decidable for finite theories:


\begin{proposition}[Finite S-reduction problem decidability]
\label{prop:sdecidable}
Let $T$ be a finite theory. Then the corresponding S-reduction problem is decidable.
\end{proposition}
\begin{proof}
We can enumerate each element $T'$ of $2^T$ in ascending order on the cardinality of $T'$. For each $T'$ we can check whether $T'$ subsumes $T$, which is decidable because subsumption between clauses is decidable. If $T'$ subsumes $T$ then we correctly return $T'$; otherwise we continue to enumerate. Because the set $2^T$ is finite the enumeration must halt. Because the set $2^T$ contains $T$ the algorithm will in the worst-case return $T$. Thus the problem is decidable.
\end{proof}
\subsection{Entailment reduction}

As mentioned in the introduction, Cropper and Muggleton \cite{crop:minmeta} previously used entailment reduction \cite{mugg:progol} to reduce sets of metarules using the notion of an entailment redundant clause:

\begin{definition}[E-redundant clause]
The clause $C$ is entailment redundant (E-redundant) in the clausal theory $T \cup \{C\}$ whenever $T\models C$.
\end{definition}


\noindent
If $T$ is a clausal theory then the pair $(T,\models)$ is an input to the \sym{}-reduction problem, which leads to the entailment reduction problem (E-reduction). We show the relationship between an E- and a S-reduction:

\begin{proposition}
\label{prop:esubs}
Let $T$ be a clausal theory, $T_S$ be a S-reduction of $T$, and $T_E$ be an E-reduction of $T$. Then $T_E \models T_S$.
\end{proposition}
\begin{proof}
Assume the opposite, i.e. $T_E \not\models T_S$.
This assumption implies that there is a clause $C \in T_S$ such that $T_E \not\models C$.
By the definition of S-reduction, $T_S$ is a subset of $T$ so $C$ must be in $T$, which implies that $T_E \not\models T$.
But this contradicts the premise that $T_E$ is an E-reduction of $T$.
Therefore the assumption cannot hold, and thus $T_E \models T_S$.
\end{proof}






\noindent
We show that the E-reduction problem is undecidable for arbitrary clausal theories:

\begin{proposition}[E-reduction problem clausal decidability]
The E-reduction problem for clausal theories is undecidable.
\end{proposition}
\begin{proof}
Follows from the undecidability of entailment in clausal logic \cite{church:problem}.
\end{proof}

\noindent
The E-reduction problem for Horn theories is also undecidable:

\begin{proposition}[E-reduction problem Horn decidability]
The E-reduction problem for Horn theories is undecidable.
\end{proposition}
\begin{proof}
Follows from the undecidability of entailment in Horn logic \cite{horn:undecidable}.
\end{proof}

\noindent
The E-reduction problem is, however, decidable for finite Datalog theories:

\begin{proposition}[E-reduction problem Datalog decidability]
\label{prop:eproblemdecidable}
The E-reduction problem for finite Datalog theories is decidable.
\end{proposition}
\begin{proof}
Follows from the decidability of entailment in Datalog \cite{dantsin:lp} using a similar algorithm to the one used in the proof of Proposition \ref{prop:sdecidable}.
\end{proof}
\subsection{Derivation reduction}

As mentioned in the introduction, entailment reduction can be too strong a form of reduction.
We therefore describe a new form of reduction based on derivability \cite{crop:dreduce,crop:sldres}.
Although our notion of derivation reduction can be defined for any proof system (such as unconstrained resolution as is done in \cite{crop:sldres}) we focus on SLD-resolution because we want to reduce sets of metarules, which are definite clauses.
We define the function $R^n(T)$ of a Horn theory $T$ as:

\begin{center}
  \begin{tabular}{l}
    $R^0(T) = T$\\
    $R^n(T) = \{C  | C_1 \in R^{n-1}(T),C_2 \in T,C$ is the binary resolvent of $C_1$ and $C_2\}$
  \end{tabular}
\end{center}





\noindent
We use this function to define the Horn closure of a Horn theory:

\begin{definition}[Horn closure]
The Horn closure $R^*(T)$ of a Horn theory $T$ is: $$\bigcup\limits_{n\in\mathbb{N}}R^n(T)$$
\end{definition}

\noindent
We state our notion of derivability:

\begin{definition}[Derivability]
A Horn clause $C$ is derivable from the Horn theory $T$, written $T \vdash C$, if and only if $C \in R^*(T)$.
\end{definition}

\noindent
We define a \emph{derivationally redundant} (D-redundant) clause:

\begin{definition}[D-redundant clause]
A clause $C$ is derivationally redundant in the Horn theory $T \cup \{C\}$ if $T \vdash C$.
\end{definition}

\noindent
Let $T$ be a Horn theory, then the pair $(T,\vdash)$ is an input to the \sym{}-reduction problem, which leads to the \emph{derivation reduction} problem (D-reduction problem). Note that a theory can have multiple D-reductions. For instance, consider the theory $T$:

\begin{center}
  \begin{tabular}{l}
    $C_1 = P(A,B) \leftarrow Q(B,A)$\\
    $C_2 = P(A,B) \leftarrow Q(A,C),R(C,B)$\\
    $C_3 = P(A,B) \leftarrow Q(C,A),R(C,B)$
  \end{tabular}
\end{center}

\noindent
One D-reduction of $T$ is $\{C_1,C_2\}$ because we can resolve the first body literal of $C_2$ with $C_1$ to derive $C_3$ (up to variable renaming). Another D-reduction of $T$ is $\{C_1,C_3\}$ because we can likewise resolve the first body literal of $C_3$ with $C_1$ to derive $C_2$.

We can show the relationship between E- and D-reductions by restating the notion of a SLD-deduction \cite{ilp:book}:

\begin{definition}[SLD-deduction \cite{ilp:book}]
Let $T$ be a Horn theory and $C$ be a Horn clause. Then there exists a SLD-deduction of $C$ from $T$, written $T \vdash_d C$, if $C$ is a tautology or if there exists a clause $D$ such that $T \vdash D$ and $D$ subsumes $C$.
\end{definition}

\noindent
We can use the \emph{subsumption theorem} \cite{ilp:book} to show the relationship between SLD-deductions and logical entailment:

\begin{theorem}[SLD-subsumption theorem \cite{ilp:book}]
Let $T$ be a Horn theory and $C$ be a Horn clause. Then $T \models C$ if and only if $T \vdash_d C$.
\end{theorem}

\noindent
We can use this result to show the relationship between an E- and a D-reduction:
\begin{proposition}
Let $T$ be a Horn theory, $T_E$ be an E-reduction of $T$, and $T_D$ be a D-reduction of $T$. Then $T_E \models T_D$.
\end{proposition}
\begin{proof}
Follows from the definitions of E-reduction and D-reduction because an E-reduction can be obtained from a D-reduction with an additional subsumption check.
\end{proof}

\noindent
We also use the SLD-subsumption theorem to show that the D-reduction problem is undecidable for Horn theories:

\begin{theorem}[D-reduction problem Horn decidability]
The D-reduction problem for Horn theories is undecidable.
\end{theorem}
\begin{proof}
Assume the opposite, that the problem is decidable, which implies that $T \vdash C$ is decidable.
Since $T \vdash C$ is decidable and subsumption between Horn clauses is decidable \cite{subsumption-npcomplete}, then finding a SLD-deduction is also decidable.
Therefore, by the SLD-subsumption theorem, entailment between Horn clauses is decidable.
However, entailment between Horn clauses is undecidable \cite{schmid-schauss}, so the assumption cannot hold.
Therefore, the problem must be undecidable.
\end{proof}



\noindent
However, the D-reduction problem is decidable for any fragment \M{a}{m} (e.g. definite Datalog clauses where each clause has at least one body literal, with additional arity and body size constraints). To show this result, we first introduce two lemmas:

\begin{lemma}
\label{lemma1}
Let $D$, $C_1$, and $C_2$ be definite clauses with $m_d$, $m_{c1}$, and $m_{c2}$ body literals respectively, where $m_d$, $m_{c1}$, and $m_{c2}$ > 0.
If $\{C_1,C_2\} \vdash D$ then $m_{c1} \leq m_{d}$ and $m_{c2} \leq m_{d}$.
\end{lemma}
\begin{proof}
Follows from the definitions of SLD-resolution \cite{ilp:book}.
\end{proof}


\noindent
Note that Lemma \ref{lemma1} does not hold for unconstrained resolution because it allows for factorisation \cite{ilp:book}. Lemma \ref{lemma1} also does not hold when facts (bodyless definite clauses) are allowed because they would allow for resolvents that are smaller in body size than one of the original two clauses.

\begin{lemma}
\label{lemma2}
Let \M{a}{m} be a fragment of metarules. Then \M{a}{m} is finite up to variable renaming.
\end{lemma}
\begin{proof}
Any literal in \M{a}{m} has at most $a$ first-order variables and 1 second-order variable, so any literal has at most $a+1$ variables.
Any metarule has at most $m$ body literals plus the head literal, so any metarule has at most $m+1$ literals.
Therefore, any metarule has at most $((a+1)(m+1))$ variables.
We can arrange the variables in at most $((a+1)(m+1))!$ ways, so there are at most $((a+1)(m+1))!$ metarules in \M{a}{m} up to variable renaming.
Thus \M{a}{m} is finite up to variable renaming.
\end{proof}

\noindent
Note that the bound in the proof of Lemma \ref{lemma2} is a worst-case result.
In practice there are fewer usable metarules because we consider fragments of constrained theories, thus not all clauses are admissible, and in all cases the order of the body literals is irrelevant. 
We use these two lemmas to show that the D-reduction problem is decidable for \M{a}{m}:

\begin{theorem}[\M{a}{m}-D-reduction problem decidability]
\label{thm:decidable}
The D-reduction problem for theories included in \M{a}{m} is decidable.
\end{theorem}
\begin{proof}
  Let $T$ be a finite clausal theory in \M{a}{m} and $C$ be a definite clause with $n>0$ body literals.
  The problem is whether $T \vdash C$ is decidable. By Lemma $\ref{lemma1}$, we cannot derive $C$ from any clause which has more than $n$ body literals.
  We can therefore restrict the resolution closure $R^*(T)$ to only include clauses with body lengths less than or equal to $n$.
  In addition, by Lemma \ref{lemma2} there are only a finite number of such clauses so we can compute the fixed-point of $R^*(T)$ restricted to clauses of size smaller or equal to $n$ in a finite amount of steps and check whether $C$ is in the set. If it is then $T \vdash C$; otherwise $T \not\vdash C$.
\end{proof}

\subsection{k-derivable clauses}

\noindent
Propositions \ref{prop:sdecidable} and \ref{prop:eproblemdecidable} and Theorem \ref{thm:decidable} show that the \sym{}-reduction problem is decidable under certain conditions. However, as we will shown in Section \ref{sec:algorithms}, even in decidable cases, solving the \sym{}-reduction problem is computationally expensive. We therefore solve restricted k-bounded versions of the E- and D-reduction problems, which both rely on SLD-derivations.
Specifically, we focus on resolution depth-limited derivations using the notion of \emph{k-derivability}:

\begin{definition}[k-derivability]
\label{def:kderiv}
Let $k$ be a natural number. Then a Horn clause $C$ is k-derivable from the Horn theory $T$, written $T \vdash_k C$, if and only if $C \in R^k(T)$.
\end{definition}



\noindent
The definitions for k-bounded E- and D-reductions follow from this definition but are omitted for brevity. In Section \ref{sec:algorithms} we introduce a general algorithm (Algorithm \ref{alg:reduce}) to solve the S-reduction problem and k-bounded E- and D-reduction problems.
\section{Reduction algorithms}
\label{sec:algorithms}

In Section \ref{sec:results} we logically reduce sets of metarules. We now describe the reduction algorithms that we use.

\subsection{\sym{}-reduction algorithm}

The {\tt reduce} algorithm (Algorithm \ref{alg:reduce}) shows a general \sym{}-reduction algorithm that solves the \sym{}-reduction problem (Definition \ref{def:sym-prob}) when the input theory is finite\footnote{In practice we use more efficient algorithms for each approach. For instance, in the derivation reduction Prolog implementation we use the knowledge gained from Lemma \ref{lemma1} to add pruning so as to ignore clauses that are too large to be useful to check whether a clause is derivable.}. We ignore cases where the input is infinite because of the inherent undecidability of the problem. Algorithm \ref{alg:reduce} is largely based on Plotkin's clausal reduction algorithm \cite{plotkin:thesis}. Given a finite clausal theory $T$ and a binary relation \sym{}, the algorithm repeatedly tries to remove a \sym{}-redundant clause in $T$. If it cannot find a \sym{}-redundant clause, then it returns the \sym{}-reduced theory.
Note that since derivation reduction is only defined over Horn theories, in a $\vdash$-reduction input $(T,\vdash)$, the theory $T$ has to be Horn.
We show total correctness of the algorithm:

\begin{proposition}[Algorithm \ref{alg:reduce} total correctness]
\label{prop:alg-reduce-proof}
Let ($T$,\sym{}) be a \sym-reduction input where $T$ is finite.
Let the corresponding \sym-reduction problem be decidable.
Then Algorithm \ref{alg:reduce} solves the \sym-reduction problem.
\end{proposition}
\begin{proof}
Trivial by induction on the size of $T$.
\end{proof}

\begin{algorithm}[ht]
\begin{myalgorithm}[]
func reduce($T$,*\sym{}*):
    if $|T| < 2$:
        return $T$
    if there is a clause $C$ in $T$ such that $T \setminus \{C\}$ *\sym{}* $C$:
        return reduce($T \setminus \{C\}$,*\sym{}*)
    else:
        return $T$
\end{myalgorithm}
\caption{\sym{}-reduction algorithm}
\label{alg:reduce}
\end{algorithm}


\noindent
Note that Proposition \ref{prop:alg-reduce-proof} assumes that the given reduction problem is decidable and that the input theory is finite. If you call Algorithm 1 with an arbitrary clausal theory and the $\models$ relation then it will not necessarily terminate. We can call Algorithm \ref{alg:reduce} with specific binary relations, where each variation has a different time-complexity. Table \ref{tab:alg1-variants} shows different ways of calling Algorithm \ref{alg:reduce} with their corresponding time complexities, where we assume finite theories as input. We show the complexity of calling Algorithm \ref{alg:reduce} with the subsumption relation:

\begin{proposition}[S-reduction complexity]
If $T$ is a finite clausal theory then calling Algorithm \ref{alg:reduce} with (T,$\preceq$) requires at most $O(|T|^3)$ calls to a subsumption algorithm.
\end{proposition}
\begin{proof}
For every clause in $T$ the algorithm checks whether any other clause in $T$ subsumes $C$ which requires at most $O(|T|^2)$ calls to a subsumption algorithm. If any clause $C$ is found to be S-redundant then the algorithm repeats the procedure on the theory ($T \setminus \{C\}$), so overall the algorithm requires at most $O(|T|^3)$ calls to a subsumption algorithm.
\end{proof}

\noindent
Note that a more detailed analysis of calling Algorithm \ref{alg:reduce} with the subsumption relation would depend on the subsumption algorithm used, which is an NP-complete problem \cite{subsumption-npcomplete}. We show the complexity of calling Algorithm \ref{alg:reduce} with the k-bounded entailment relation:

\begin{proposition}[k-bounded E-reduction complexity]
\label{prop:ecomplexity}
If $T$ is a finite Horn theory and $k$ is a natural number then calling Algorithm \ref{alg:reduce} with (T,$\models_k$) requires at most $O(|T|^{k+2})$ resolutions.
\end{proposition}
\begin{proof}
In the worst case the derivation check (line 4) requires searching the whole SLD-tree which has a maximum branching factor $|T|$ and a maximum depth $k$ and takes $O(|T|^{k})$ steps. The algorithm potentially does this step for every clause in $T$ so the complexity of this step is $O(|T|^{k+1})$. The algorithm has to perform this check for every clause in $T$ with an overall worst-case complexity $O(|T|^{k+2})$.
\end{proof}

\noindent
The complexity of calling Algorithm \ref{alg:reduce} with the k-derivation relation is identical:

\begin{proposition}[k-bounded D-reduction complexity]
Let $T$ be a finite Horn theory and $k$ be a natural number then calling Algorithm \ref{alg:reduce} with (T,$\vdash_k$) requires at most $O(|T|^{k+2})$ resolutions.
\end{proposition}
\begin{proof}
Follows using the same reasoning as Proposition \ref{prop:ecomplexity}.
\end{proof}

\begin{table}
\centering
\normalsize
\begin{tabular}{|l|l|l|}
\hline
Relation & Output & Complexity     \\
\hline
$\preceq$          & S-reduction & $O(|T|^3)$           \\
$\models$           & E-reduction & Undecidable    \\
$\models_k$           & k-E-reduction & $O(|T|^{k+2})$    \\
$\vdash$           & D-reduction & Undecidable    \\
$\vdash_k$         & k-D-reduction & $O(|T|^{k+2})$\\
\hline
\end{tabular}
\caption{Outputs and complexity of Algorithm \ref{alg:reduce} for different input relations and an arbitrary finite clausal theory $T$. The time complexities are a function of the size of the given theory, denoted by $|T|$.}
\label{tab:alg1-variants}
\end{table}

\subsection{\M{a}{m}-\sym{}-reduction algorithm}

Although Algorithm \ref{alg:reduce} solves the \sym{}-reduction problem, it does not solve the \M{a}{m}-reduction problem (Definition \ref{def:hamprob}). For instance, suppose you have the following theory $T$ in \M{2}{4}:

\begin{center}
  \begin{tabular}{l}
    $M_1 = P(A,B) \leftarrow Q(B,A)$\\
    $M_2 = P(A,B) \leftarrow Q(A,A),R(B,B)$\\
    $M_3 = P(A,B) \leftarrow Q(A,C),R(B,C)$\\
    $M_4 = P(A,B) \leftarrow Q(B,C),R(A,D),S(A,D),T(B,C)$
  \end{tabular}
\end{center}

\noindent
Suppose you want to know whether $T$ can be E-reduced to \M{2}{2}.
Then calling Algorithm \ref{alg:reduce} with ($T,\models$) (i.e. the entailment relation) will return $T' = \{M_1,M_4\}$ because:
$M_4 \models M_2$\footnote{
Rename the variables in $M_4$ to form $M_4' = P_0(X_1,X_2) \leftarrow P_1(X_2,X_3),P_2(X_1,X_4),P_3(X_1,X_4),$ $P_4(X_2,X_3)$.
Then $M_4' \theta = P(A,B) \leftarrow R(B,B),Q(A,A),Q(A,A),R(B,B)$ where $\theta=\{P_0/P,P_1/R,P_2/Q,P_3/Q,$ $P_4/R,X_1/A,X_2/B,X_3/B,X_4/A\}$.
It follows that $M_4' \theta \subseteq M_2$, so $M_4 \preceq M_2$, which in turn implies $M_4 \models M_2$.
}, $M_4 \models M_3$\footnote{
Rename the variables in $M_4$ to form $M_4' = P_0(X_1,X_2) \leftarrow P_1(X_2,X_3),P_2(X_1,X_4),P_3(X_1,X_4),$ $P_4(X_2,X_3)$.
Then $M_4' \theta = P(A,B) \leftarrow R(B,C),Q(A,C),Q(A,C),R(B,C)$ where $\theta=\{P_0/P,P_1/R,P_2/Q,P_3/Q,$ $P_4/R,X_1/A,X_2/B,X_3/C,X_4/C\}$.
It follows that $M_4' \theta \subseteq M_3$, so $M_4 \preceq M_3$, which in turn implies $M_4 \models M_3$.
}, and $\{M_1,M_4\}$ cannot be further E-reduced.

Although $T'$ is an E-reduction of $T$, it is not in \M{2}{2} because $M_4$ is not in \M{2}{2}.
However, the theory $T$ can be \M{2}{2}-E-reduced to $\{M_1,M_2,M_3\}$ because $\{M_2,M_3\} \models M_4$\footnote{
Rename the variables in $M_3$ to form $M_3' = P_0(X,Y) \leftarrow P_1(X,Z),P_2(Y,Z)$.
Resolve the first body literal of $M_2$ with $M_3$ to form $R_1 = P(A,B) \leftarrow P_1(A,Z),P_2(A,Z),R(B,B)$.
Rename the variables $P_1$ to $P_3$, $P_2$ to $P_4$, and $Z$ to $Z_1$ in $R_1$ (to standardise apart the variables) to form $R_2 = P(A,B) \leftarrow P_3(A,Z_1),P_4(A,Z_1),R(B,B)$.
Resolve the last body literal of $R_2$ with $M_3'$ to form $R_3 = P(A,B) \leftarrow P_3(A,Z_1),P_4(A,Z_1),P_1(B,Z),P_2(B,Z)$.
Rename the variables $Z_1$ to $D$, $Z to C$, $P_3$ to $R$, $P_4$ to $S$, $P_1$ to $Q$, and $P_2$ to $T$ in $R_3$ to form $R_4 = P(A,B) \leftarrow R(A,D),S(A,D),Q(B,C),T(B,C)$.Thus, $R_4 = M_4$, so it follows that and $\{M_2,M_3\} \models M_4$
},
and $\{M_1,M_2,M_3\}$ cannot be further reduced.
In general, let $T$ be a theory in \M{a}{m} and an $T'$ be an E-reduction of $T$, then $T'$ is not necessarily in \M{a}{2}.

Algorithm \ref{alg:reduce2} overcomes this limitation of Algorithm \ref{alg:reduce}. Given a finite clausal theory $T$, a binary relation \sym{}, and a reduction fragment \M{a}{m}, Algorithm \ref{alg:reduce2} determines whether there is a \sym{}-reduction of $T$ in \M{a}{m}. If there is, it returns the reduced theory; otherwise it returns false. In other words, Algorithm \ref{alg:reduce2} solves the \M{a}{m}-\sym{}-reduction problem. We show total correctness of Algorithm \ref{alg:reduce2}:

\begin{proposition}[Algorithm \ref{alg:reduce2} correctness]
Let ($T$,\sym{},\M{a}{m}) be a \M{a}{m}-\sym-reduction input. If the corresponding \sym-reduction problem is decidable then Algorithm \ref{alg:reduce2} solves the corresponding \M{a}{m}-\sym-reduction problem.
\end{proposition}
\begin{sketchproof}
We provide a sketch proof for brevity. We need to show that the function \texttt{aux} correctly determines whether $B$ \sym{} $T$, which we can show by induction on the size of $T$. Assuming \texttt{aux} is correct, then if $T$ can be reduced to $B$, the \texttt{mreduce} function calls Algorithm \ref{alg:reduce} to reduce $B$, which is correct by Proposition \ref{prop:alg-reduce-proof}. Otherwise it returns false.
\end{sketchproof}

\begin{algorithm}[ht]
\begin{myalgorithm}[]
function mreduce($T$,*\sym{}*,*\M{a}{m}*)
    $B = \{C | C\in T \cap\M{a}{m}\}$
    if aux($T$,*\sym{}*,$B$):
        return reduce($B$,*\sym{}*)
    return false

function aux($T$,*\sym{}*,$B$)
    if |T| == 0:
        return true
    pick any $C$ in $T$
    $T'$ = $T \setminus \{C\}$
    if $B$ *\sym{}* $C$:
        return aux($T'$,*\sym{}*,$B$)
    return false
\end{myalgorithm}
    \caption{\M{a}{m}-\sym{}-reduction}
    \label{alg:reduce2}
\end{algorithm}
\section{Reduction of metarules}
\label{sec:results}

We now logically reduce fragments of metarules.
Given a fragment \M{a}{m} and a reduction operator \sym{}, we have three main goals:

\begin{description}
\item[\textbf{G1:}] identify a \M{a}{k}-\sym{}-reduction of \M{a}{m} for some $k$ as small as possible
\item[\textbf{G2:}] determine whether \M{a}{2} \sym{} \M{a}{\infty}
\item[\textbf{G3:}] determine whether \M{a}{\infty} has any (finite) \sym{}-reduction
\end{description}

\noindent
We work on these goals for fragments of \M{a}{m} relevant to ILP.
Table \ref{tab:fragments} shows the four fragments and their main restrictions.
The subsequent sections precisely describe the fragments.

Our first goal (\textbf{G1}) is to essentially minimise the number of body literals in a set of metarules, which can be seen as trying to enforce an Occamist bias.
We are particularly interested reducing sets of metarules to fragments with at most two body literals because \M{\{2\}}{2} augmented with one function symbol has universal Turing machine expressivity \cite{tarnlund:hornclause}.
In addition, previous work on MIL has almost exclusively used metarules from the fragment \M{2}{2}.
Our second goal (\textbf{G2}) is more general and concerns reducing an infinite set of metarules to \M{a}{2}.
Our third goal (\textbf{G3}) is similar, but is about determining whether an infinite set of metarules has any finite reduction.

We work on the goals by first applying the reduction algorithms described in the previous section to finite fragments restricted to 5 body literals (i.e. \M{a}{5}).
This value gives us a sufficiently large set of metarules to reduce but not too large that the reduction problem is intractable.
When running the E- and D-reduction algorithms (both k-bounded), we use a resolution-depth bound of 7, which is the largest value for which the algorithms terminate in reasonable time\footnote{The entailment and derivation reduction algorithms often took 4-5 hours to find a reduction.
However, in some cases, typically where the fragments contained many metarules, the algorithms took around 12 hours to find a reduction.
By contrast, the subsumption reduction algorithm typically found a reduction in 30 minutes.}.
After applying the reduction algorithms to the finite fragments, we then try to solve \textbf{G2} by extrapolating the results to the infinite case (i.e. \M{a}{\infty}).
In cases where \M{a}{2} \nsym{} \M{a}{\infty}, we then try to solve \textbf{G3} by seeing whether there exists any natural number $k$ such that \M{a}{k} \sym{} \M{a}{\infty}.

\begin{table}
\centering
\normalsize
\begin{tabular}{|c|l|}
\hline
Fragment & Description\\
\hline
\C{}{} & connected clauses\\
\D{}{} & connected Datalog clauses\\
\K{}{} & connected Datalog clauses without singleton variables\\
\U{}{} & connected Datalog clauses without duplicate variables\\
\hline
\end{tabular}
\caption{The four main fragments of \M{}{} that we consider.}
\label{tab:fragments}
\end{table}

\subsection{Connected (\C{a}{m}) results}
\label{sec:connected}

We first consider a general fragment of metarules.
The only constraint is that we follow the standard ILP convention \cite{crop:minmeta,ilp:book,evans:dilp,gottlob:complexity} and focus on connected clauses\footnote{Connected clauses are also known as linked clauses \cite{gottlob:complexity}.}:

\begin{definition}[Connected clause]
A clause is connected if the literals in the clause cannot be partitioned into two sets such that the variables appearing in the literals of one set are disjoint from the variables appearing in the literals of the other set.
\end{definition}

\noindent
The following clauses are all connected:
\begin{center}
  \begin{tabular}{l}
$P(A) \leftarrow Q(A)$\\
$P(A,B) \leftarrow Q(A,C)$\\
$P(A,B) \leftarrow Q(A,B),R(B,D),S(D,B)$
  \end{tabular}
\end{center}

\noindent
By contrast, these clauses are not connected:
\begin{center}
  \begin{tabular}{l}
$P(A) \leftarrow Q(B)$\\
$P(A,B) \leftarrow Q(A),R(C)$\\
$P(A,B) \leftarrow Q(A,B),S(C)$
  \end{tabular}
\end{center}

\noindent
We denote the connected fragment of \M{a}{m} as \C{a}{m}.
Table \ref{tab:cfrags} shows the maximum body size and the cardinality of the reductions obtained when applying the reduction algorithms to \C{a}{5} for different values of $a$.
To give an idea of the scale of the reductions, the fragment \C{\{1,2\}}{5} contains 77398 unique metarules, of which E-reduction removed all but two of them.
Table \ref{tab:12-c} shows the actual reductions for \C{\{1,2\}}{5}.
Reductions for other connected fragments are in Appendix \ref{app:connected-reductions}.

\begin{table}[ht]
\centering
\begin{tabular}{|c|c|c|c|c|c|c|}
\hline
Arities            & \multicolumn{2}{c|}{S-reduction} & \multicolumn{2}{c|}{E-reduction} & \multicolumn{2}{c|}{D-reduction}    \\
\hline
$a$                   & Bodysize       & Cardinality   & Bodysize       & Cardinality   & Bodysize            & Cardinality \\
\hline
0         & 1    & 1              & 1    & 1              & 2        & 2            \\
1         & 1    & 1              & 1    & 1              & 2         & 2            \\
2         & 1    & 4              & 1    & 1              & 5         & 6            \\
0,1   & 1    & 3              & 1    & 3              & 2         & 5            \\
0,2   & 1    & 6              & 1    & 3              & 5   & 21            \\
1,2   & 1    & 9              & 1    & 2              & 4   & 8            \\
0,1,2 & 1    & 12             & 1    & 4              & 5 & 13\\
\hline
\end{tabular}
\caption{Cardinality and maximal body size of the reductions of \C{a}{5}. All the fragments can be S- and E-reduced to \C{a}{1} but they cannot all be D-reduced to \C{a}{2}.}
\label{tab:cfrags}
\end{table}

\begin{table}[ht]
\scriptsize
\centering
    \begin{tabular}[t]{|c|c|c|}
    \hline
    S-reduction & E-reduction & D-reduction\\
    \hline
    \begin{tabular}[t]{l}
$P(A) \leftarrow Q(A)$\\
$P(A) \leftarrow Q(A,B)$\\
$P(A) \leftarrow Q(B,A)$\\
$P(A,B) \leftarrow Q(A)$\\
$P(A,B) \leftarrow Q(B)$\\
$P(A,B) \leftarrow Q(A,C)$\\
$P(A,B) \leftarrow Q(B,C)$\\
$P(A,B) \leftarrow Q(C,A)$\\
$P(A,B) \leftarrow Q(C,B)$\\
    \end{tabular}
    &
    \begin{tabular}[t]{l}
$P(A) \leftarrow Q(B,A)$\\
$P(A,B) \leftarrow Q(A)$
    \end{tabular}
    &
    \begin{tabular}[t]{l}
$P(A) \leftarrow Q(B,A)$\\
$P(A,A) \leftarrow Q(B,A)$\\
$P(A,B) \leftarrow Q(B)$\\
$P(A,B) \leftarrow Q(B,A)$\\
$P(A,B) \leftarrow Q(B,B)$\\
$P(A,B) \leftarrow Q(A,B),R(A,B)$\\
$P(A,B) \leftarrow Q(A,C),R(B,C)$\\
$P(A,B) \leftarrow Q(A,C),R(A,D),S(B,C),T(B,D),U(C,D)$
    \end{tabular}\\
    \hline
    \end{tabular}
\caption{Reductions of the connected fragment \C{\{1,2\}}{5}.}
\label{tab:12-c}
\end{table}

\noindent
As Table \ref{tab:cfrags} shows, all the fragments can be S- and E-reduced to \C{a}{1}.
We show that in general \C{a}{\infty} has a \C{a}{1}-S-reduction:

\begin{theorem}[\C{a}{\infty} S-reducibility]
  \label{thm:c-sreduce}
  For all $a>0$, the fragment \C{a}{\infty} has a \C{a}{1}-S-reduction.
\end{theorem}
\begin{proof}
Let $C$ be any clause in \C{a}{\infty}, where $a>0$.
By the definition of connected clauses there must be at least one body literal in $C$ that shares a variable with the head literal of $C$.
The clause formed of the head of $C$ with the body literal directly connected to it is by definition in \C{a}{1} and clearly subsumes $C$.
Therefore \C{a}{1} $\preceq$ \C{a}{\infty}.
\end{proof}

\noindent
We likewise show that \C{a}{\infty} always has a \C{a}{1}-E-reduction:

\begin{theorem}[\C{a}{\infty} E-reducibility]
\label{thm:c-ereduce}
For all $a>0$, the fragment \C{a}{\infty} has a \C{a}{1}-E-reduction.
\end{theorem}
\begin{proof}
Follows from Theorem \ref{thm:c-sreduce} and Proposition \ref{prop:esubs}.
\end{proof}

\noindent
As Table \ref{tab:cfrags} shows, the fragment \C{2}{5} could not be D-reduced to \C{2}{2} when running the derivation reduction algorithm.
However, because we run the derivation reduction algorithm with a maximum derivation depth, this result alone is not enough to guarantee that the output cannot be further reduced.
Therefore, we show that \C{2}{5} cannot be D-reduced to \C{2}{2}:

\begin{proposition}[\C{2}{5} D-irreducibility]
\label{prop:c25-irreducible}
The fragment \C{2}{5} has no \C{2}{2}-D-reduction.
\end{proposition}
\begin{proof}
  We denote by $\perm{C}$ the set of all clauses that can be obtained from a given clause $C$ by permuting the arguments in its literals up to variable renaming.
  For example if $C=P(A,B)\leftarrow Q(A,C)$ then $\perm{C}=\{(C),(P(A,B)\leftarrow Q(C,A)),(P(B,A)\leftarrow Q(A,C)),(P(B,A)\leftarrow Q(C,A))\}$ up to variable renaming.

  Let $C_I$ denote the clause $P(A,B) \leftarrow Q(A,C),R(A,D),S(B,C),T(B,D),U(C,D)$.
  We prove that no clause in $\perm{C_I}$ can be derived from \C{2}{2} by induction on the length of derivations.
  Formally, we show that there exist no derivations of length $n$ from \C{2}{2} to a clause in \perm{C_I}.
  We reason by contradiction and w.l.o.g.\ we consider only the clause $C_I$.

  For the base case $n=0$, assume that there is a derivation of length $0$ from \C{2}{2} to $C_I$.
  This assumption implies that $C_I\in\C{2}{2}$, but this clearly cannot hold given the body size of $C_I$.

  For the general case, assume that the property holds for all $k<n$ and by contradiction consider the final inference in a derivation of length $n$ of $C_I$ from \C{2}{2}.
  Let $C_1$ and $C_2$ denote the premises of this inference.
  Then the literals occurring in $C_I$ must occur up to variable renaming in at least one of $C_1$ and $C_2$.
  We consider the following cases separately.
  \begin{itemize}
  \item
  All the literals of $C_I$ occur in the same premise: because of Lemma
  \ref{lemma1}, this case is impossible because this premise would contain more literals than $C_I$ (the ones from $C_I$ plus the resolved literal).

  \item
  Only one of the literals of $C_I$ occurs separately from the others: w.l.o.g., assume that the literal $Q(A,C)$ occurs alone in $C_2$ (up to variable renaming).
  Then $C_2$ must be of the form $H(A,C)\leftarrow Q(A,C)$ or $H(C,A)\leftarrow Q(A,C)$ for some $H$, where the $H$-headed literal is the resolved literal of the inference that allows the unification of $A$ and $C$ with their counterparts in $C_1$\footnote{
  Those are the only options to derive $C_I$.Otherwise, e.g. with $C_2 = H(A',C') \leftarrow Q(A',D')$, the resulting clause is not $C_I$ because $D'$ is not unified with any of the variables in $C_1$ (whereas $A'$ unifies with $A$ and $C'$ with $C$), e.g. the result includes the literal $Q(A,D')$ instead of $Q(A,C)$ hence it is not $C_I$.
  }.
  In this case, $C_1$ belongs to $\perm{C_I}$ and a derivation of $C_1$ from \C{2}{2} of length smaller than $n$ exists as a strict subset of the derivation to $C_I$ of length $n$.
  This contradicts the induction hypothesis, thus the assumed derivation of $C$ cannot exist.

  \item Otherwise, the split of the literals of $C_I$ between $C_1$ and $C_2$ is always such that at least three variables must be unified during the inference.
  For example, consider the case where $P(A,B) \leftarrow Q(A,C) \subset C_1$ and the set $\{R(A',D),S(B',C'),T(B',D),U(C',D)\}$ occurs in the body of $C_2$ (up to variable renaming).
  Then $A'$, $B'$ and $C'$ must unify respectively with $A$, $B$ and $C$ for $C_I$ to be derived (up to variable renaming).
  However the inference can at most unify two variable pairs since the resolved literal must be dyadic at most and thus this inference is impossible, a contradiction.
  \end{itemize}
  Thus $C_I$ and all of \perm{C_I} cannot be derived from \C{2}{2}.
  Note that, since \perm{C_I} is also neither a subset of \C{2}{3} nor of \C{2}{4}, this proof also shows that \perm{C_I} cannot be derived from \C{2}{3} and from \C{2}{4}.
\end{proof}

\noindent
We generalise this result to \C{2}{\infty}:

\begin{theorem}[\C{2}{\infty} D-irreducibility]
\label{prop:c2-irreducible}
The fragment \C{2}{\infty} has no D-reduction.
\label{thm:irreduicble}
\end{theorem}
\begin{proof}
  It is enough to prove that \C{2}{\infty} does not have a \C{2}{m}-D-reduction for an arbitrary $m$ because any D-reduced theory, being finite, admits a bound on the body size of the clauses it contains.
  Starting from $C_I$ as defined in the proof of Proposition \ref{prop:c25-irreducible}, apply the following transformation iteratively for $k$ from 1 to $m$: replace the literals containing $Q$ and $R$ (i.e. at first $Q(A,C)$ and $R(A,D)$) with the following set of literals $Q(A,C_k)$, $R(A,D_k)$, $V_k(C_k,D_k)$, $Q_k(C_k,C)$, $R_k(D_k,D)$ where all variables and predicate variables labeled with $k$ are new.
  Let the resulting clause be denoted $C_{I_m}$.
  This clause is of body size $3m+5$ and thus does not belong to \C{2}{m}.
  Moreover, for the same reason that $C_I$ cannot be derived from any \C{2}{m'} with $m'<5$ (see the proof of Proposition \ref{prop:c25-irreducible}) $C_{I_m}$ cannot be derived from any \C{2}{m'} with $m'<3m+5$.
  In particular, $C_{I_m}$ cannot be derived from \C{2}{m}.
\end{proof}

\noindent
Another way to generalise Proposition \ref{prop:c25-irreducible} is the following:

\begin{theorem}[\C{a}{\infty} D-irreducibility]
  \label{thm:ca-irreducible}
  For $a\geq 2$, the fragment \C{a}{\infty} has no \C{a}{a^2+a-2}-D-reduction.
\end{theorem}
\begin{proof}
  Let $C_a$ denote the clause
  \begin{equation*}
    \begin{split}
    C_a = P(A_1,\dots,A_a)\leftarrow &Q_{1,1}(A_1,B_{1,1},\dots,B_{1,a-1})\dots Q_{1,a}(A_1,B_{a,1},\dots,B_{a,a-1})\\
    & \dots\\
    & Q_{a,1}(A_a,B_{1,1},\dots,B_{1,a-1})\dots Q_{a,a}(A_a,B_{a,1},\dots,B_{a,a-1})\\
    & R_1(B_{1,1},\dots,B_{a,1}),\dots,R_{a-1}(B_{1,a-1},\dots,B_{a,a-1})
    \end{split}
  \end{equation*}
  Note that for $a = 2$, the clauses $C_a$ and $C_I$ from the proof of Proposition \ref{prop:c25-irreducible} coincide.
  In fact to show that $C_a$ is irreducible for any $a$, it is enough to consider the proof of Proposition \ref{prop:c25-irreducible} where $C_a$ is substituted to $C_I$ and where the last case is generalised in the following way:
  \begin{itemize}
  \item the split of the literals of $C_a$ between $C_1$ and $C_2$ is always such that at least $a+1$ variables must be unified during the inference, which is impossible since the resolved literal can at most hold $a$ variables.
  \end{itemize}
  The reason this proof holds is that any subset of $C_a$ contains at least $a+1$ distinct variables.
  Since $C_a$ is of body size $a^2+a-1$, this counter-example proves that \C{a}{\infty} has no \C{a}{a^2+a-2}-D-reduction.
\end{proof}
Note that this is enough to conclude that \C{a}{\infty} cannot be reduced to \C{a}{2} but it does not prove that \C{a}{\infty} is not D-reducible.



\subsubsection{Summary}
Table \ref{tab:csummary} summarises our theoretical results from this section.
Theorems \ref{thm:c-sreduce} and \ref{thm:c-ereduce} show that \C{a}{\infty} can always be S- and E-reduced to \C{a}{1} respectively.
By contrast, Theorem \ref{prop:c2-irreducible} shows that \C{2}{\infty} cannot be D-reduced to \C{2}{2}.
In fact, Theorem \ref{prop:c2-irreducible} says that \C{2}{\infty} has no D-reduction.
Theorem \ref{prop:c2-irreducible} has direct (negative) implications for MIL systems such as Metagol and HEXMIL.
We discuss these implications in more detail in Section \ref{sec:conclusions}.

\begin{table}[ht]
\centering
\normalsize
\begin{tabular}{|c|ccc|}
\hline
Arity & S & E & D\\
\hline
1   & \y{} & \y{} & \y{} \\
2   & \y{} & \y{} & \n{} \\
>2 & \y{} & \y{} & \n{}\\
\hline
\end{tabular}
\caption{
  Existence of a S-, E- or D-reduction of \C{a}{\infty} to \C{a}{2}.
  The symbol \y{} denotes that the fragment does have a reduction.
  The symbol \n{} denotes that the fragment does not have a reduction.
  }
\label{tab:csummary}
\end{table}
\subsection{Datalog (\D{a}{m}) results}

We now consider Datalog clauses, which are often used in ILP \cite{mugg:metagold,crop:metafunc,hexmil,albarghouthi2017constraint,ALPS,evans:dilp}. The relevant Datalog restriction is that if a variable appears in the head of a clause then it must also appear in a body literal. If we look at the S-reductions of \C{\{1,2\}}{5} in Table \ref{tab:12-c} then the clause $P(A,B) \leftarrow Q(B)$ is not a Datalog clause because the variable $A$ appears in the head but not in the body. We denote the Datalog fragment of \C{a}{m} as \D{a}{m}. Table \ref{tab:dfrags} shows the results of applying the reduction algorithms to \D{a}{5} for different values of $a$. Table \ref{tab:12-d} shows the reductions for the fragment \D{\{1,2\}}{5}, which are used in Experiment 3 (Section \ref{sec:egames}) to induce Datalog game rules from observations. Reductions for other Datalog fragments are in Appendix \ref{app:Datalog-reductions}.

\begin{table}[ht]
\centering
\begin{tabular}{|c|c|c|c|c|c|c|}
\hline
Arities            & \multicolumn{2}{c|}{S-reduction} & \multicolumn{2}{c|}{E-reduction} & \multicolumn{2}{c|}{D-reduction}    \\
\hline
$a$                   & Bodysize       & Cardinality   & Bodysize       & Cardinality   & Bodysize            & Cardinality \\
\hline
0         & 1    & 1              & 1    & 1              & 2        & 2            \\
1         & 1    & 1              & 1    & 1              & 2         & 2            \\
2         & 2    & 4              & 2    & 2              & 5         & 10            \\
0,1   & 1    & 2              & 1    & 2              & 2         & 5            \\
0,2   & 2    & 5              & 2    & 3              & 5   & 38            \\
1,2   & 2    & 10              & 2    & 3              & 5   & 11            \\
0,1,2 & 2    & 11             & 2    & 4              & 5 & 14\\
\hline
\end{tabular}
\caption{Cardinality and maximal body size of the reductions of \D{a}{5}. All the fragments can be S- and E-reduced to \D{a}{2} but they cannot all be D-reduced to \D{a}{2}.}
\label{tab:dfrags}
\end{table}

\begin{table}[ht]
\scriptsize
\centering
    \begin{tabular}[t]{|c|c|c|}
    \hline
    S-reduction & E-reduction & D-reduction\\
    \hline
    \begin{tabular}[t]{l}
$P(A) \leftarrow Q(A)$\\
$P(A) \leftarrow Q(A,B)$\\
$P(A) \leftarrow Q(B,A)$\\
$P(A,A) \leftarrow Q(B,A)$\\
$P(A,B) \leftarrow Q(A),R(B)$\\
$P(A,B) \leftarrow Q(A),R(B,C)$\\
$P(A,B) \leftarrow Q(A,B)$\\
$P(A,B) \leftarrow Q(B),R(A,C)$\\
$P(A,B) \leftarrow Q(B,A)$\\
$P(A,B) \leftarrow Q(B,C),R(A,D)$\\
    \end{tabular}
    &
    \begin{tabular}[t]{l}
$P(A) \leftarrow Q(A,B)$\\
$P(A,B) \leftarrow Q(B,A)$\\
$P(A,B) \leftarrow Q(A),R(B)$\\
    \end{tabular}
    &
    \begin{tabular}[t]{l}
$P(A) \leftarrow Q(B,A)$\\
$P(A,A) \leftarrow Q(A)$\\
$P(A,A) \leftarrow Q(A,A)$\\
$P(A,B) \leftarrow Q(B,A)$\\
$P(A,B) \leftarrow Q(A,B),R(A,B)$\\
$P(A,B) \leftarrow Q(A,C),R(B,C)$\\
$P(A,B) \leftarrow Q(B,C),R(A,D)$\\
$P(A,B) \leftarrow Q(B,C),R(A,D),S(B,D),T(C,E)$\\
$P(A,B) \leftarrow Q(A,C),R(A,D),S(B,C),T(B,D),U(C,D)$\\
$P(A,B) \leftarrow Q(B,C),R(A,D),S(C,E),T(B,F),U(D,F)$\\
$P(A,B) \leftarrow Q(B,C),R(B,D),S(C,E),T(A,F),U(D,F)$
    \end{tabular}\\
    \hline
    \end{tabular}
\caption{Reductions of the Datalog fragment \D{\{1,2\}}{5}.}
\label{tab:12-d}
\end{table}

\noindent
We show that \D{2}{\infty} can be S-reduced to \D{2}{2}:

\begin{proposition}[\D{2}{\infty} S-reducibility]
  \label{prop:d22-sreduce}
  The fragment \D{2}{\infty} has a \D{2}{2}-S-reduction.
\end{proposition}
\begin{proof}
  Follows using the same argument as in Theorem \ref{thm:c-sreduce} but the reduction is to \D{2}{2} instead of \D{2}{1}.
  This difference is due to the Datalog constraint that states: if a variable appears in the head it must also appear in the body.
  For clauses with dyadic heads, if the two head argument variables occur in two distinct body literals then the clause cannot be further reduced beyond \D{2}{2}.
\end{proof}

\noindent
We show how this result cannot be generalised to \D{a}{\infty}:

\begin{theorem}[\D{a}{\infty} S-irreducibility]
  \label{thm:da-nosreduce}
  For $a>0$, the fragment \D{a}{\infty} does not have a \D{a}{a-1}-S-reduction.
\end{theorem}
\begin{proof}
As a counter-example to a \D{a}{a-1}-S-reduction, consider $C_a=P(X_1,\dots,X_a)\leftarrow Q_1(X_1),$ $\dots,Q_a(X_a)$.
The clause $C_a$ does not belong to \D{a}{a-1} and cannot be S-reduced to it because the removal of any subset of its literals leaves argument variables in the head without their counterparts in the body.
Hence, any subset of $C_a$ does not belong to the Datalog fragment.
Thus $C_a$ cannot be subsumed by a clause in \D{a}{a-1}.
\end{proof}

\noindent
However, we can show that \D{a}{\infty} can always be S-reduced to \D{a}{a}:

\begin{theorem}[\D{a}{\infty} to \D{a}{a} S-reducibility]
  \label{prop:daa-sreduce}
  For $a>0$, the fragment \D{a}{\infty} has a \D{a}{a}-S-reduction.
\end{theorem}
\begin{proof}
To prove that \D{a}{\infty} has a \D{a}{a}-S-reduction it is enough to remark that any clause in \D{a}{\infty} has a subclause of body size at most $a$ that is also in \D{a}{\infty}, the worst case being clauses such as $C_a$ where all argument variables in the head occur in a distinct literal in the body.
\end{proof}

\noindent
We also show that \D{a}{\infty} always has a \D{a}{2}-E-reduction, starting with the following lemma:

\begin{lemma}
\label{lemmaeresd}
For $a>0$ and $n\in \{1,\dots,a\}$, the clause \[P_0(A_1,A_2,\dots,A_n) \leftarrow P_1(A_1), P_2(A_2), \dots, P_n(A_n)\] is \D{a}{2}-E-reducible.
\end{lemma}
\begin{proof}
By induction on $n$.
\begin{itemize}
\item For the base case $n=2$, by definition \D{a}{2} contains $P_0(A_1,A_2) \leftarrow P_1(A_1), P_2(A_2)$
\item
For the inductive step, assume the claim holds for $n-1$.
We show it holds for $n$.
By definition \D{a}{2} contains the clause $D_1=P(A_1,A_2,\dots,A_{n}) \leftarrow P_0(A_1,A_2,\dots,A_{n-1}), P_n(A_{n})$.
By the inductive hypothesis, $D_2=P_0(A_1,A_2,\dots,A_{n-1})\leftarrow P_1(A_1),\dots,P_{n-1}(A_{n-1})$ is \D{a-1}{2}-E-reducible, and thus also \D{a}{2}-E-reducible.
Together, $D_1$ and $D_2$ entail $D=P_0(A_1,A_2,\dots,A_n) \leftarrow P_1(A_1), P_2(A_2), \dots,$ $ P_n(A_n)$, which can be seen by resolving the literal $P_0(A_1,A_2,\dots,A_{n-1})$ from $D_1$ with the same literal from $D_2$ to derive $D$.
Thus D is \D{a}{2}-E-reducible.
\end{itemize}
\end{proof}

\begin{theorem}[\D{a}{\infty} E-reducibility]
\label{thm:d-ereduce}
For $a>0$, the fragment \D{a}{\infty} has a \D{a}{2}-E-reduction.
\end{theorem}
\begin{proof}
Let $C$ be any clause in \D{a}{\infty}.
We denote the head of $C$ by $P(A_1,\dots,A_n)$, where $0<n\leq a$.
The possibility that some of the $A_i$ are equal does not impact the reasoning.

If $n=1$, then by definition, there exists a literal $L_1$ in the body of $C$ such that $A_1$ occurs in $L_1$.
It is enough to consider the clause $P(A_1)\leftarrow L_1$ to conclude, because $P(A_1)$ is the head of $C$ and $L_1$ belongs to the body of $C$, thus $P(A_1)\leftarrow L_1$ entails $C$, and this clause belongs to \D{a}{2}.

In the case where $n>1$, there must exist literals $L_1,\dots,L_n$ in the body of $C$ such that $A_i$ occurs in $L_i$ for $i\in\{1,\dots,n\}$.
Consider the clause $C' = P(A_1,\dots,A_n) \leftarrow L_1,\dots,L_n$.
There are a few things to stress about $C'$:
\begin{itemize}
\item The clause $C'$ belongs to \D{a}{\infty}. 
\item Some $L_i$ may be identical with each other, since the $A_i$s may occur together in literals or simply be equal, but this scenario does not impact the reasoning.
\item The clause $C'$ entails $C$ because $C'$ is equivalent to a subset of $C$ (but this subset may be distinct from $C'$ due to $C'$ possibly including some extra duplicated literals).
\end{itemize}
Now consider the clause $D = P(A_1,\dots,A_n)\leftarrow P_1(A_1),\dots,P_n(A_n)$.
For $i\in\{1,\dots,n\}$, the clause $P_i(A_i)\leftarrow L_i$ belongs to \D{a}{2} by definition, thus $\D{a}{2}\cup D\vdash D'$ where $D'=P(A_1,\dots,A_n)\leftarrow L_1,\dots,L_n$.
Moreover, by Lemma \ref{lemmaeresd}, $D$ is \D{a}{2}-E-reducible, hence $D'$ is also \D{a}{2}-E-reducible.
Note that this notation hides the fact that if a variable occurs in distinct body literals $L_i$ in $C'$, this connection is not captured in $D'$ where distinct variables will occur instead, thus there is no guarantee that $D'=C'$.
For example, if $C'=P(A_1,A_2)\leftarrow Q(A_1,B,A_2),R(A_2,B)$ then $D'=P(A_1,A_2)\leftarrow Q(A_1,B,A_2'),Q(A_1,B,A_2'),R(A_2,B'),R(A_2,B')$
However, it always holds that $D'\models C'$, because $D'$ subsumes $C'$.
In our small example, it is enough to consider the substitution $\theta=\{B'/B,A_2'/A_2\}$ to observe this.
Thus by transitivity of entailment, we can conclude that $C$ is \D{a}{2}-E-reducible.

\end{proof}

\noindent
As Table \ref{tab:dfrags} shows, not all of the fragments can be D-reduced to \D{a}{2}.
In particular, the result that \D{2}{\infty} has no \D{2}{2}-D-reduction follows from Theorem \ref{prop:c2-irreducible} because the counterexamples presented in the proof also belong to \D{2}{\infty}.


\subsubsection{Summary}
Table \ref{tab:dsummary} summarises our theoretical results from this section.
Theorem \ref{thm:da-nosreduce} shows that \D{a}{\infty} never has a \D{a}{a-1}-S-reduction.
This result differs from the connected fragment where \C{a}{\infty} could always be S-reduced to \C{a}{2}.
However, Theorem \ref{thm:da-nosreduce} shows that \D{a}{\infty} can always be S-reduced to \D{a}{a}.
As with the connected fragment, Theorem \ref{thm:d-ereduce} shows that \D{a}{\infty} can always be E-reduced to \C{a}{2}.
The result that \D{2}{\infty} has no D-reduction follows from Theorem \ref{prop:c2-irreducible}.


\begin{table}[ht]
\centering
\normalsize
\begin{tabular}{|c|ccc|}
\hline
Arity & S & E & D\\
\hline
1   & \y{} & \y{} & \y{} \\
2   & \y{} & \y{} & \n{} \\
>2 & \n{} & \y{} & \n{}\\
\hline
\end{tabular}
\caption{
  Existence of a S-, E- or D-reduction of \D{a}{\infty} to \D{a}{2}.
  }
\label{tab:dsummary}
\end{table}

\subsection{Singleton-free (\K{a}{m}) results}
\label{sec:singletonfree}

It is common in ILP to require that all the variables in a clause appear at least twice \cite{raedt:clint,crop:minmeta,mugg:golem}, which essentially eliminates singleton variables. We call this fragment the \emph{singleton-free} fragment:

\begin{definition}[Singleton-free]
A clause is singleton-free if each first-order variable appears at least twice
\end{definition}

\noindent
For example, if we look at the E-reductions of the connected fragment \C{\{1,2\}}{5} shown in Table \ref{tab:12-c} then the clause $P(A) \leftarrow Q(B,A)$ is not singleton-free because the variable $B$ only appears once.
We denote the singleton-free fragment of \D{a}{m} as \K{a}{m}. Table \ref{tab:kfrags} shows the results of applying the reduction algorithms to \K{a}{5}.
Table \ref{fig:2-s} shows the reductions of \K{\{2\}}{5}.
Reductions for other singleton-free fragments are in Appendix \ref{app:singleton-reductions}.

\begin{table}[ht]
\centering
\begin{tabular}{|c|c|c|c|c|c|c|}
\hline
Arities            & \multicolumn{2}{c|}{S-reduction} & \multicolumn{2}{c|}{E-reduction} & \multicolumn{2}{c|}{D-reduction}    \\
\hline
$a$                    & Bodysize       & Cardinality   & Bodysize       & Cardinality   & Bodysize            & Cardinality \\
\hline
0         & 1    & 1              & 1    & 1              & 2        & 2            \\
1         & 1    & 1              & 1    & 1              & 2         & 2            \\
2         & 4    & 3              & 2    & 3              & 5         & 7            \\
0,1   & 1    & 2              & 1    & 2              & 2         & 5            \\
0,2   & 5    & 4              & 2    & 3              & 5   & 23            \\
1,2   & 4    & 8              & 2    & 4              & 5   & 8            \\
0,1,2 & 4    & 9             & 2    & 5              & 5 & 11\\
\hline
\end{tabular}
\caption{Cardinality and maximal body size of the reductions of \K{a}{5}.}
\label{tab:kfrags}
\end{table}

\begin{table}[ht]
\footnotesize
\centering
    \begin{tabular}[t]{|c|c|c|}
    \hline
    S-reduction & E-reduction & D-reduction\\
    \hline
    \begin{tabular}[t]{l}
$P(A,B) \leftarrow Q(A,B)$\\
$P(A,B) \leftarrow Q(B,A)$\\
$P(A,B) \leftarrow Q(B,C),R(A,D),$\\
$\hspace{4.5em} S(A,D),T(B,C)$
    \end{tabular}
    &
    \begin{tabular}[t]{l}
$P(A,B) \leftarrow Q(B,A)$\\
$P(A,B) \leftarrow Q(A,A),R(B,B)$\\
$P(A,B) \leftarrow Q(A,C),R(B,C)$\\
    \end{tabular}
    &
    \begin{tabular}[t]{l}
$P(A,A) \leftarrow Q(A,A)$\\
$P(A,B) \leftarrow Q(B,A)$\\
$P(A,A) \leftarrow Q(A,B),R(B,B)$\\
$P(A,B) \leftarrow Q(A,A),R(B,B)$\\
$P(A,B) \leftarrow Q(A,B),R(A,B)$\\
$P(A,B) \leftarrow Q(A,C),R(B,C)$\\
$P(A,B) \leftarrow Q(A,C),R(A,D),$\\
$\hspace{4.5em} S(B,C),T(B,D),U(C,D)$

    \end{tabular}\\
    \hline
    \end{tabular}
\caption{Reductions of the singleton-free fragment \K{\{2\}}{5}}
\label{fig:2-s}
\end{table}

\noindent
Unlike in the connected and Datalog cases, the fragment \K{\{2\}}{5} is no longer S-reducible to \K{\{2\}}{2}.
We show that \K{2}{\infty} cannot be reduced to \K{2}{2}.

\begin{proposition}[\K{2}{\infty} S-reducibility]
  \label{prop:ksred}
  The fragment \K{2}{\infty} does not have a \K{2}{2}-S-reduction.
\end{proposition}
\begin{proof}
As a counter-example, consider the clause:

\[C=P(A,B) \leftarrow Q(A,D), R(A,D), S(B,C), T(B,C)\]

\noindent
Consider removing any non-empty subset of literals from the body of $C$.
Doing so leads to a singleton variable in the remaining clause, so it is not a singleton-free clause.
Moreover, for any other clause to subsume $C$ it must be more general than $C$, but that is not possible again because of the singleton-free constraint\footnote{Note that this proof also shows that \K{2}{\infty} does not have a \K{2}{3}-S-reduction.}.
\end{proof}

\noindent
We can likewise show that this result holds in the general case:

\begin{theorem}[\K{a}{\infty} S-reducibility]
  \label{thm:sirreducible}
  For $a\geq 2$, the fragment \K{a}{\infty} does not have a \K{a}{2a-1}-S-reduction.
\end{theorem}
\begin{proof}
  We generalise the clause $C$ from the proof of Proposition \ref{prop:ksred} to define the clause $C_a = P(A_1,\dots,A_a)\leftarrow P_1(A_1,B_1),P_2(A_1,B_1),\dots,P_{2a-1}(A_a,B_a),P_{2a}(A_a,B_a)$.
  The same reasoning applies to $C_a$ as to $C (= C_2)$, making $C_a$ irreducible in \K{a}{\infty}.
  Moreover $C_a$ is of body size $2a$, thus $C_a$ is a counterexample to a \K{a}{2a-1}-S-reduction of \K{a}{\infty}.
\end{proof}

\noindent
However, all the fragments can be E-reduced to \K{a}{2}.

\begin{theorem}[\K{a}{\infty} E-reducibility]
\label{thm:s-ereduce}
For $a>0$, the fragment \K{a}{\infty} has a \K{a}{2}-E-reduction.
\end{theorem}

\begin{proof}
  The proof of Theorem \ref{thm:s-ereduce} is an adaptation of that of Theorem \ref{thm:d-ereduce}.
  The only difference is that if $n=1$ then $P(A_1)\leftarrow L_1,L_1$ must be considered instead of $P(A_1)\leftarrow L_1$ to ensure the absence of singleton variables in the body of the clause, and for the same reason, in the general case, the clause $D'=P(A_1,\dots,A_n)\leftarrow L_1,...,L_n$ must be replaced by $D'=P(A_1,\dots,A_n)\leftarrow L_1,L_1,\dots,L_n,L_n$.
  Note that $C'$ is not modified and thus may or may not belong to \K{a}{\infty}.
  However, it is enough that $C'\in\D{a}{\infty}$.
  With these modifications, the proof carries from \K{a}{\infty} to \K{a}{2} as from \D{a}{\infty} to \D{a}{2}, including the results in Lemma \ref{lemmaeresd}.

\end{proof}

\subsubsection{Summary}
Table \ref{tab:ssummary} summarises our theoretical results from this section.
Theorem \ref{thm:sirreducible} shows that for $a\geq 2$, the fragment \K{a}{\infty} does not have a \K{a}{2a-1}-S-reduction.
This result contrasts with the Datalog fragment where \D{a}{\infty} always has a \D{a}{a}-S-reduction.
As is becoming clear, adding more restrictions to a fragment typically results in less S-reducibility.
By contrast, as with the connected and Datalog fragments, Theorem \ref{thm:s-ereduce} shows that fragment \K{a}{\infty} always has a \K{a}{2}-E-reduction.
In addition, as with the other fragments, \K{a}{\infty} has no D-reduction for $a\geq 2$.

\begin{table}[ht]
\centering
\normalsize
\begin{tabular}{|c|ccc|}
\hline
Arity & S & E & D\\
\hline
1   & \y{} & \y{} & \y{} \\
2   & \n{} & \y{} & \n{} \\
>2 & \n{} & \y{} & \n{}\\
\hline
\end{tabular}
\caption{
  Existence of a S-, E- or D-reduction of \K{a}{\infty} to \K{a}{2}.
  }
\label{tab:ssummary}
\end{table}
\subsection{Duplicate-free (\U{a}{m}) results}

The previous three fragments are general in the sense that they have been widely used in ILP.
By contrast, the final fragment that we consider is of particular interest to MIL.
Table \ref{tab:metarules} shows a selection of metarules commonly used in the MIL literature.
These metarules have been successfully used despite no theoretical justification.
However, if we consider the reductions of the three fragments so far, the \emph{identity}, \emph{precon}, and \emph{postcon} metarules do not appear in any reduction.
These metarules can be derived from the reductions, typically using either the $P(A) \leftarrow Q(A,A)$ or $P(A,A) \leftarrow Q(A)$ metarules.
To try to identify a reduction which more closely matches the metarules shown in Table \ref{tab:metarules}, we consider a fragment that excludes clauses in which a literal contains multiple occurrences of the same variable.
For instance, this fragment excludes the previously mentioned metarules and also excludes the metarule $P(A,A) \leftarrow Q(B,A)$, which was in the D-reduction shown in Table \ref{tab:12-c}.
We call this fragment \emph{duplicate-free}.
It is a sub-fragment of \K{a}{m} and we denote it as \U{a}{m}.

Table \ref{tab:12-u} shows the reductions for the fragment \U{\{1,2\}}{5}.
Reductions for other duplicate-free fragments are in Appendix \ref{app:duplicate-reductions}.
As Table \ref{tab:12-u} shows, the D-reduction of \U{\{1,2\}}{5} contains some metarules commonly used in the MIL literature.
For instance, it contains the \emph{identity$_1$}, \emph{didentity$_2$}, and \emph{precon} metarules.
We use the metarules shown in Table \ref{tab:12-u} in Experiments 1 and 2 (Sections \ref{sec:etrains} and \ref{sec:estrings}) to learn Michalski trains solutions and string transformation programs respectively.

Table \ref{tab:ufrags} shows the results of applying the reduction algorithms to \U{a}{5} for different values of $a$.
All the theoretical results that hold for the singleton-free fragments hold similarly for the duplicate-free fragments for the following reasons:
\begin{itemize}
  \item (S) The clauses in the proofs of Proposition \ref{prop:ksred} and Theorem \ref{thm:sirreducible} belong to \U{a}{\infty}.
  \item (E) If the clause $C$ considered initially in the proof of Theorem \ref{thm:s-ereduce} belongs to \U{a}{\infty}, then all the subsequent clauses in that proof are also duplicate-free.
\item (D) In the proof of Theorem \ref{thm:irreduicble}, the $C_{I_m}$ family of clauses all belong to \U{a}{\infty}.
\end{itemize}
Thus Table \ref{tab:ssummary} is also a summary of the S-, E- and D-reduction results of \U{a}{\infty} to \U{a}{2}.


\begin{table}[ht]
\scriptsize
\centering
    \begin{tabular}[t]{|c|c|c|}
    \hline
    S-reduction & E-reduction & D-reduction\\
    \hline
    \begin{tabular}[t]{l}
$P(A) \leftarrow Q(A)$\\
$P(A) \leftarrow Q(A,B),R(A,B)$\\
$P(A,B) \leftarrow Q(A,B)$\\
$P(A,B) \leftarrow Q(B,A)$\\
$P(A,B) \leftarrow Q(A),R(B)$\\
$P(A,B) \leftarrow Q(B),R(A,C),S(A,C)$\\
$P(A,B) \leftarrow Q(A),R(B,C),S(B,C)$\\
$P(A,B) \leftarrow Q(B,C),R(A,D),$\\
$\hspace{4.5em} S(A,D),T(B,C)$
    \end{tabular}
    &
    \begin{tabular}[t]{l}
$P(A) \leftarrow Q(A,B),R(A,B)$\\
$P(A,B) \leftarrow Q(B,A)$\\
$P(A,B) \leftarrow Q(A),R(B)$\\
    \end{tabular}
    &
    \begin{tabular}[t]{l}
$P(A) \leftarrow Q(A)$\\
$P(A) \leftarrow Q(A),R(A)$\\
$P(A) \leftarrow Q(A,B),R(B)$\\
$P(A) \leftarrow Q(A,B),R(A,B)$\\
$P(A,B) \leftarrow Q(B,A)$\\
$P(A,B) \leftarrow Q(A),R(B)$\\
$P(A,B) \leftarrow Q(A),R(A,B)$\\
$P(A,B) \leftarrow Q(A,B),R(A,B)$\\
$P(A,B) \leftarrow Q(A,C),R(B,C)$\\
$P(A,B) \leftarrow Q(A,C),R(A,D),S(B,C),$\\
$\hspace{4.5em} T(B,D),U(C,D)$\\
$P(A,B) \leftarrow Q(B,C),R(A,D),S(B,D)$\\
$\hspace{4.5em} T(C,E),U(E)$\\
$P(A,B) \leftarrow Q(B,C),R(A,D),S(B,D),$\\
$\hspace{4.5em} T(C,E),U(C,E)$
    \end{tabular}\\
    \hline
    \end{tabular}
\caption{Reductions of the fragment \U{\{1,2\}}{5}}
\label{tab:12-u}
\end{table}

\begin{table}[ht]
\centering
\begin{tabular}{|c|c|c|c|c|c|c|}
\hline
Arities            & \multicolumn{2}{c|}{S-reduction} & \multicolumn{2}{c|}{E-reduction} & \multicolumn{2}{c|}{D-reduction}    \\
\hline
$a$                    & Bodysize       & Cardinality   & Bodysize       & Cardinality   & Bodysize            & Cardinality \\
\hline
0         & 1    & 1              & 1    & 1              & 2        & 2            \\
1         & 1    & 1              & 1    & 1              & 2         & 2            \\
2         & 4    & 3              & 5    & 2              & 5         & 10            \\
0,1   & 1    & 2              & 1    & 2              & 2         & 5            \\
0,2   & 5    & 4              & 5    & 3              & 5   & 38            \\
1,2   & 4    & 8              & 2    & 3              & 5   & 12            \\
0,1,2 & 4    & 9             & 2    & 4              & 5 & 16\\
\hline
\end{tabular}
\caption{Cardinality and body size of the reductions of \U{a}{5}.}
\label{tab:ufrags}
\end{table}



\subsection{Summary}

We started this section with three goals (\textbf{G1}, \textbf{G2}, and \textbf{G3}).
Table \ref{tab:summary} summarises the results towards these goals for fragments of metarules relevant to ILP (Table \ref{tab:fragments}).
For \textbf{G1}, our results are mostly empirical, i.e. the results are the outputs of the reduction algorithms.
For \textbf{G2}, Table \ref{tab:summary} shows that the results are all positive for E-reduction, but mostly negative for S- and D-reduction, especially for Datalog fragments.
Similarly, for \textbf{G3} the results are again positive for E-reduction but negative for S- and D-reduction for Datalog fragments.
We discuss the implications of these results in Section \ref{sec:conclusions}.

\begin{table}[ht]
\centering
\normalsize
\begin{tabular}{c|ccc|ccc|ccc|ccc}

 Arities  & \multicolumn{3}{c}{\C{a}{\infty}}
  & \multicolumn{3}{c}{\D{a}{\infty}}
  & \multicolumn{3}{c}{\K{a}{\infty}}
  & \multicolumn{3}{c}{\U{a}{\infty}}\\
  \hline
$a$ & S & E & D
& S & E & D
& S & E & D
& S & E & D \\


1   & \y{} & \y{} & \y{}
    & \y{} & \y{} & \y{}
    & \y{} & \y{} & \y{}
    & \y{} & \y{} & \y{} \\

2   & \y{} & \y{} & \n{}
    & \y{} & \y{} & \n{}
    & \n{} & \y{} & \n{}
    & \n{} & \y{} & \n{} \\





>2 & \y{} & \y{} & \n{}
    & \n{} & \y{} & \n{}
    & \n{} & \y{} & \n{}
    & \n{} & \y{} & \n{}
\end{tabular}
\caption{
  Existence of a S-, E- or D-reduction of \M{a}{\infty} to \M{a}{2}.
  The symbol \y{} denotes that the fragment does have such a reduction.
  The symbol \n{} denotes that the fragment does not have such a reduction.
  }
\label{tab:summary}
\end{table}
\section{Experiments}
\label{sec:experiments}

As explained in Section \ref{sec:intro}, deciding which metarules to use for a given learning task is a major open problem.
The problem is the trade-off between efficiency and expressivity: the hypothesis space grows given more metarules (Theorem \ref{thm:hypspace}), so we wish to use fewer metarules, but if we use too few metarules then we lose expressivity.
In this section we experimentally explore this trade-off.
As described in Section \ref{sec:related}, Cropper and Muggleton \cite{crop:minmeta} showed that learning with E-reduced sets of metarules can lead to higher predictive accuracies and lower learning times compared to learning with non-E-reduced sets.
However, as argued in Section \ref{sec:intro}, we claim that E-reduction is not always the most suitable form of reduction because it can remove metarules necessary to learn programs with the appropriate specificity.
To test this claim, we now conduct experiments that compare the learning performance of Metagol 2.3.0\footnote{https://github.com/metagol/metagol/releases/tag/2.3.0}, the main MIL implementation, when given different reduced sets of metarules\footnote{Experimental data is available at http://github.com/andrewcropper/mlj19-reduce
}.
We test the null hypothesis:

\begin{description}
\item[\textbf{Null hypothesis 1}] There is no difference in the learning performance of Metagol when using different reduced sets of metarules
\end{description}

\noindent
To test this null hypothesis, we consider three domains: Michalski trains, string transformations, and game rules.

\subsection{Michalski trains}
\label{sec:etrains}

In the Michalski trains problems \cite{michalski:trains} the task is to induce a program that distinguishes eastbound trains from westbound trains. Figure \ref{fig:trains-prog} shows an example target program, where the target concept (\tw{f/1}) is that the train has a long carriage with two wheels and another with three wheels.


\begin{figure}[ht]
\centering
\normalsize
\begin{tabular}{|c|}
\hline
\begin{lstlisting}
f(X):-
    has_car(X,C1),
    long(C1),
    two_wheels(C1),
    has_car(X,C2),
    long(C2),
    three_wheels(C2).
\end{lstlisting}\\
\hline
\end{tabular}
\caption{
An example Michalski trains target program.
In the Michalski trains domain, a carriage (\tw{car}) can be long or short.
A short carriage always has two wheels.
A long carriage has either two or three wheels.
}
\label{fig:trains-prog}
\end{figure}

\subsubsection{Materials}
To obtain the experimental data, we first generated 8 random target train programs where the programs are progressively more difficult, where difficulty is measured by the number of literals in the generated program from the easiest task T$_1$ to the most difficult task T$_8$.
Figure \ref{fig:trains-bk} shows the background predicates available to Metagol.
We vary the metarules given to Metagol.
We use the S-, E-, and D-reductions of the fragment \U{\{1,2\}}{5} (Table \ref{tab:12-u}).
In addition, we also consider the \U{\{1,2\}}{2} fragment of the D-reduction of \U{\{1,2\}}{5}, i.e. a subset of the D-reduction consisting only of metarules with at most two body literals.
This fragment, which we denote as $D^{*}$, contains three fewer metarules than the D-reduction of \U{\{1,2\}}{5}.
Table \ref{tab:d2fragment} shows this fragment.

\begin{figure}[ht]
\centering
\normalsize
\begin{tabular}{|ll|}
\hline
\texttt{has\_car/2} & \texttt{has\_load/2}\\
\texttt{short/1} & \texttt{long/1}\\
\texttt{two\_wheels/1} & \texttt{three\_wheels/1}\\
\texttt{roof\_open/1} & \texttt{roof\_closed/1}\\
\texttt{zero\_load/1} & \texttt{one\_load/1}\\
\texttt{two\_load/1} & \texttt{circle/1}\\
\texttt{triangle/1} & \texttt{rectangle/1}\\
\hline
\end{tabular}
\caption{Background relations available in the trains experiment.}
\label{fig:trains-bk}
\end{figure}

\begin{table}[ht]
\centering
\normalsize
\begin{tabular}[t]{|ll|}
\hline
$P(A) \leftarrow Q(A)$ & $P(A,B) \leftarrow Q(B,A)$\\
$P(A) \leftarrow Q(A),R(A)$ & $P(A,B) \leftarrow Q(A),R(B)$\\
$P(A) \leftarrow Q(A,B),R(B)$ & $P(A,B) \leftarrow Q(A),R(A,B)$\\
$P(A) \leftarrow Q(A,B),R(A,B)$ & $P(A,B) \leftarrow Q(A,B),R(A,B)$\\
& $P(A,B) \leftarrow Q(A,C),R(B,C)$\\
\hline
\end{tabular}
\caption{The D$^*$ fragment, which is the D-reduction of the fragment \U{\{1,2\}}{5} restricted to the fragment \U{\{1,2\}}{2}.}
\label{tab:d2fragment}
\end{table}

\subsubsection{Method}
For each train task $t_i$ in $\{T_1,\dots,T_8\}$:
\begin{enumerate}
    \item Generate 10 training examples of $t_i$, half positive and half negative
    \item Generate 200 testing examples of $t_i$, half positive and half negative
    \item For each set of metarules $m$ in the S-, E-, D-, and $D^*$-reductions:
    \begin{enumerate}
        \item Learn a program for task $t_i$ using the training examples and metarules $m$
        \item Measure the predictive accuracy of the learned program using the testing examples
    \end{enumerate}
\end{enumerate}

\noindent
If a program is not found in 10 minutes then no program is returned and every testing example is deemed to have failed. We measure mean predictive accuracies, mean learning times, and standard errors over 10 repetitions.

\begin{table}[ht]
\centering
\normalsize
\begin{tabular}{c|c|c|c|c}
\textbf{Task} &
\textbf{S} &
\textbf{E} &
\textbf{D} &
\textbf{D$^*$}\\
\hline
$T_1$ & \textbf{100} $\pm$ 0 & \textbf{100} $\pm$ 0 & \textbf{100} $\pm$ 0 & \textbf{100} $\pm$ 0\\
$T_2$ & \textbf{100} $\pm$ 0 & \textbf{100} $\pm$ 0 & \textbf{100} $\pm$ 0 & \textbf{100} $\pm$ 0\\
$T_3$ & 68 $\pm$ 5 & 62 $\pm$ 5 & \textbf{100} $\pm$ 0 & \textbf{100} $\pm$ 0\\
$T_4$ & 75 $\pm$ 6 & 75 $\pm$ 6 & \textbf{100} $\pm$ 0 & \textbf{100} $\pm$ 0\\
$T_5$ & 92 $\pm$ 4 & 78 $\pm$ 6 & 78 $\pm$ 6 & \textbf{100} $\pm$ 0\\
$T_6$ & 52 $\pm$ 2 & 50 $\pm$ 0 & 70 $\pm$ 6 & \textbf{100} $\pm$ 0\\
$T_7$ & 95 $\pm$ 3 & 65 $\pm$ 5 & 82 $\pm$ 5 & \textbf{100} $\pm$ 0\\
$T_8$ & 55 $\pm$ 3 & 52 $\pm$ 2 & 72 $\pm$ 6 & \textbf{98} $\pm$ 2\\
\hline
mean & 80 $\pm$ 1  & 73 $\pm$ 2  & 88 $\pm$ 2  & \textbf{100} $\pm$ 0
\end{tabular}
\caption{Predictive accuracies when using different reduced sets of metarules on the Michalski trains problems.}
\label{fig:train-accs}
\end{table}

\begin{table}[ht]
\centering
\normalsize
    \begin{tabular}{c|c|c|c|c}
\textbf{Task} &
\textbf{S} &
\textbf{E} &
\textbf{D} &
\textbf{D$^*$}\\
\hline

T1 & \textbf{0} $\pm$ 0 & \textbf{0} $\pm$ 0 & \textbf{0} $\pm$ 0 & \textbf{0} $\pm$ 0\\
T2 & \textbf{0} $\pm$ 0 & \textbf{0} $\pm$ 0 & \textbf{0} $\pm$ 0 & \textbf{0} $\pm$ 0\\
T3 & 424 $\pm$ 59 & 461 $\pm$ 56 & \textbf{0} $\pm$ 0 & \textbf{0} $\pm$ 0\\
T4 & 322 $\pm$ 64 & 340 $\pm$ 61 & \textbf{0} $\pm$ 0 & \textbf{0} $\pm$ 0\\
T5 & 226 $\pm$ 48 & 320 $\pm$ 59 & 361 $\pm$ 59 & \textbf{5} $\pm$ 2\\
T6 & 583 $\pm$ 17 & 600 $\pm$ 0 & 429 $\pm$ 51 & \textbf{7} $\pm$ 2\\
T7 & 226 $\pm$ 44 & 446 $\pm$ 55 & 243 $\pm$ 61 & \textbf{6} $\pm$ 1\\
T8 & 550 $\pm$ 35 & 570 $\pm$ 30 & 361 $\pm$ 64 & \textbf{183} $\pm$ 40\\
\hline
mean & 292 $\pm$ 16  & 342 $\pm$ 17  & 174 $\pm$ 16  & \textbf{25} $\pm$ 5
\end{tabular}
\caption{Learning times in seconds when using different reduced sets of metarules on the Michalski trains problems. Note that the values are rounded, so 0 represents that a solution was found in under half a second.}
\label{fig:train-times}
\end{table}

\begin{figure}[ht]
\centering
\normalsize

\begin{tabular}{|c|}
\hline
\begin{lstlisting}
f(A):-has_car(A,B),f1(A,B).
f1(A,B):-three_wheels(B),has_car(A,C),f2(A,C).
f2(A,B):-roof_open(B),has_car(A,C),has_car(A,C).
\end{lstlisting} \\
\hline
S program\\
\hline
\end{tabular}

\vspace{5mm}

\begin{tabular}{|c|}
\hline
\begin{lstlisting}
f(A):-has_car(A,B),f1(A,B).
f1(A,B):-f2(A),three_wheels(B).
f2(A):-has_car(A,B),has_car(A,B).
\end{lstlisting} \\
\hline
E program\\
\hline
\end{tabular}

\vspace{5mm}

\begin{tabular}{|c|}
\hline
\begin{lstlisting}
f(A):-f1(A),f2(A).
f1(A):-has_car(A,B),roof_open(B).
f2(A):-has_car(A,B),three_wheels(B).
\end{lstlisting} \\
\hline
D program\\
\hline
\end{tabular}

\vspace{5mm}

\begin{tabular}{|c|}
\hline
\begin{lstlisting}
f(A):-f1(A),f2(A).
f1(A):-has_car(A,B),three_wheels(B).
f2(A):-has_car(A,B),f3(B).
f3(A):-long(A),two_wheels(A).
\end{lstlisting} \\
\hline
D$^*$ program\\
\hline
\end{tabular}

\caption{Example programs learned by Metagol when varying the metarule set. The target program is shown in Figure \ref{fig:trains-prog}.}
\label{fig:T8-programs}
\end{figure}

\subsubsection{Results}
Table \ref{fig:train-accs} shows the predictive accuracies when learning with the different sets of metarules.
The $D$ set generally outperforms the $S$ and $E$ sets with a higher mean accuracy of 88\% vs 80\% and 73\% respectively.
Moreover, the $D^*$ set easily outperforms them all with a mean accuracy of 100\%.
A McNemar's test\footnote{A statistical test on paired \emph{nominal} data https://en.wikipedia.org/wiki/McNemar\%27s\_test} on the $D$ and $D^*$ accuracies confirmed the significance at the $p < 0.01$ level.

Table \ref{fig:train-times} shows the corresponding learning times when using different reduces sets of metarules.
The $D$ set outperforms (has lower mean learning time) the $S$ and $E$ sets, and again the $D^*$ set outperforms them all.
A paired t-test\footnote{A statistical test on paired \emph{ordinal} data http://www.biostathandbook.com/pairedttest.html} on the $D$ and $D^*$ learning times confirmed the significance at the $p < 0.01$ level.

The $D^*$ set performs particularly well on the more difficult tasks.
The poor performance of the $S$ and $E$ sets on the more difficult tasks is for one of two reasons.
The first reason is that the S- and E-reduction algorithms have removed the metarules necessary to express the target concept.
This observation strongly corroborates our claim that E-reduction can be too strong because it can remove metarules necessary to specialise a clause.
The second reason is that the S- and E-reduction algorithms produce sets of metarules that are still sufficient to express the target theory but doing so requires a much larger and more complex program, measured by the number of clauses needed.

The performance discrepancy between the $D$ and $D^*$ sets of metarules can be explained by comparing the hypothesis spaces searched.
For instance, when searching for a program with 3 clauses, Theorem \ref{thm:hypspace} shows that when using the $D$ set of metarules the hypothesis space contains approximately $10^{24}$ programs.
By contrast, when using the $D^*$ set of metarules the hypothesis space contains approximately $10^{14}$ programs.
As explained in Section \ref{sec:mil}, assuming that the target hypothesis is in both hypothesis spaces, the Blumer bound \cite{blumer:bound} tells us that searching the smaller hypothesis space will result in less error, which helps to explain these empirical results.
Of course, there is the potential for the $D^*$ set to perform worse than the $D$ set when the target theory requires the three removed metarules, but we did not observe this situation in this experiment.

Figure \ref{fig:T8-programs} shows the target program for T$_8$ and example programs learned by Metagol using the various reduced sets of metarules.
Only the D$^*$ program is success set equivalent\footnote{The success set of a logic program $P$ is the set of ground atoms $\{A \in hb(P)|P\cup\{ \neg A \}\;\text{has a SLD-refutation}\}$, where $hb(P)$ represents the Herband base of the logic program $P$. The success set restricted to a specific predicate symbol $p$ is the subset of the success set restricted to atoms containing the predicate symbol $p$.} to the target program when restricted to the target predicate \tw{f/1}.
In all three cases Metagol discovered that if a carriage has three wheels then it is a long carriage, i.e. Metagol discovered that the literal \tw{long(C2)} is redundant in the target program.
Indeed, if we unfold the D$^*$ program to remove the invented predicates then the resulting single clause program is one literal shorter than the target program.

Overall, the results from this experiment suggest that we can reject the null hypothesis, both in terms of predictive accuracies and learning times.

\subsection{String transformations}
\label{sec:estrings}

In \cite{mugg:metabias} and \cite{crop:metaopt} the authors evaluate Metagol on 17 real-world string transformation tasks using a predefined (hand-crafted) set of metarules. In this experiment, we compare learning with different metarules on an expanded dataset with 250 string transformation tasks.

\subsubsection{Materials}
Each string transformation task has 10 examples. Each example is an atom of the form $f(x,y)$ where $f$ is the task name and $x$ and $y$ are strings. Figure \ref{fig:p06} shows task \emph{p6} where the goal is to learn a program that filters the capital letters from the input. We supply Metagol with dyadic background predicates, such as \tw{tail}, \tw{dropLast}, \tw{reverse}, \tw{filter\_letter}, \tw{filter\_uppercase}, \tw{dropWhile\_not\_letter}, \tw{takeWhile\_uppercase}. The full details can be found in the code repository. We vary the metarules given to Metagol. We use the S-, E-, and D-reductions of the fragment \U{\{2\}}{5}. We again also use the D-reduction of the fragment \U{\{2\}}{5} restricted to the fragment \U{\{2\}}{2}, which is again denoted as $D^*$.

\begin{figure}[ht]
\centering
\normalsize
\begin{tabular}{l|l}
\textbf{Input} & \textbf{Output} \\ \hline
Arthur Joe Juan & AJJ\\
Jose Larry Scott & JLS\\
Kevin Jason Matthew & KJM\\
Donald Steven George & DSG\\
Raymond Frank Timothy & RFT
\end{tabular}
\caption{Examples of the \emph{p6} string transformation problem input-output pairs.}
\label{fig:p06}
\end{figure}

\subsubsection{Method}
Our experimental method is:
\begin{enumerate}
    \item Sample 50 tasks $Ts$ from the set $\{p1,\dots,p250\}$
    \item For each $t \in Ts$:
    \begin{enumerate}
        \item Sample 5 training examples and use the remaining examples as testing examples
        \item For each set of metarules $m$ in the S-, E-, D, and $D^*$-reductions:
        \begin{enumerate}
            \item Learn a program $p$ for task $t$ using the training examples and metarules $m$
            \item Measure the predictive accuracy of $p$ using the testing examples
        \end{enumerate}
    \end{enumerate}
\end{enumerate}

\noindent
If a program is not found in 10 minutes then no program is returned and every testing example is deemed to have failed. We measure mean predictive accuracies, mean learning times, and standard errors over 10 repetitions.

\subsubsection{Results}

Table \ref{fig:string-results} shows the mean predictive accuracies and learning times when learning with the different sets of metarules.
Note that we are not interested in the absolute predictive accuracy, which is limited by factors such as the low timeout and insufficiency of the BK.
We are instead interested in the relative accuracies.
Table \ref{fig:string-results} shows that the $D$ set outperforms the $S$ and $E$ sets, with a higher mean accuracy of 33\%, vs 22\% and 22\% respectively.
The $D^*$ set outperforms them all with a mean accuracy of 56\%.
A McNemar's test on the $D$ and $D^*$ accuracies confirmed the significance at the $p < 0.01$ level.

Table \ref{fig:string-results} shows the corresponding learning times when varying the metarules. Again, the $D$ set outperforms the $S$ and $E$ sets, and again the $D^*$ set outperforms them all.
A paired t-test on the $D$ and $D^*$ learning times confirmed the significance at the $p < 0.01$ level.

Overall, the results from this experiment give further evidence to reject the null hypothesis, both in terms of predictive accuracies and learning times.

\begin{table}[ht]
\centering
\normalsize
\begin{tabular}{c|c|c|c|c}
&
\textbf{S} &
\textbf{E} &
\textbf{D} &
\textbf{D$^*$}\\
\hline
Mean predictive accuracy (\%)  & 22 $\pm$ 0  & 22 $\pm$ 0  & 32 $\pm$ 0  & \textbf{56 $\pm$ 1}\\
Mean learning time (seconds) & 467 $\pm$ 1  & 467 $\pm$ 1  & 407 $\pm$ 3  & \textbf{270 $\pm$ 3}
\end{tabular}
\caption{Experimental results on the string transformation problems.}
\label{fig:string-results}
\end{table}
\subsection{Inducing game rules}
\label{sec:egames}

The general game playing (GGP) framework~\cite{genesereth} is a system for evaluating an agent's general intelligence across a wide range of tasks. In the GGP competition, agents are tested on games they have never seen before. In each round, the agents are given the rules of a new game. The rules are described symbolically as a logic program. The agents are given a few seconds to think, to process the rules of the game, and to then start playing, thus producing game traces. The winner of the competition is the agent who gets the best total score over all the games. In this experiment, we use the IGGP dataset \cite{iggp} which inverts the GGP task: an ILP system is given game traces and the task is to learn a set of rules (a logic program) that could have produced these traces.

\subsubsection{Materials}

The IGGP dataset contains problems drawn from 50 games.
We focus on the eight games shown in Figure \ref{fig:games} which contain BK compatible with the metarule fragments we consider (i.e. the BK contains predicates in the fragment \M{2}{m}).
The other games contain predicates with arity greater than two.
Each game has four target predicates \tw{legal}, \tw{next}, \tw{goal}, and \tw{terminal}, where the arities depend on the game. Figure \ref{fig:mindec} shows the target solution for the \tw{next} predicate for the \emph{minimal decay} game. Each game contains training/validate/test data, composed of sets of ground atoms, in a 4:1:1 split. We vary the metarules given to Metagol. We use the S-, E-, and D-reductions of the fragment \D{\{1,2\}}{5}. We again also use the D-reduction of the fragment \D{\{1,2\}}{5} restricted to the fragment \D{\{1,2\}}{2}, which is again denoted as $D^*$.

\begin{figure}[ht]
\centering
\normalsize
\begin{tabular}{|ll|}
\hline
GT attrition & GT chicken\\
GT prisoner & Minimal decay\\
Minimal even & Multiple buttons and lights\\
Scissors paper stone & Untwisty corridor\\
\hline
\end{tabular}
\caption{IGGP games used in the experiments.}
\label{fig:games}
\end{figure}

\begin{figure}[ht]
\centering
\normalsize
\begin{tabular}{|c|}
\hline
\begin{lstlisting}
next_value(X):-
    true_value(Y),
    succ(X,Y),
    does_player(noop).
next_value(5):-
    does_player(pressButton).
\end{lstlisting}\\
\hline
\end{tabular}
\caption{Target solution for the \tw{next} predicate for the \emph{minimal decay} game.}
\label{fig:mindec}
\end{figure}

\subsubsection{Method}

The majority of game examples are negative. We therefore use \emph{balanced accuracy} to evaluate the approaches. Given background knowledge $B$, sets of positive $E^+$ and negative $E^-$ testing examples, and a logic program $H$, we define the number of positive examples as $p=|E^+|$, the number of negative examples as $n=|E^-|$, the number of true positives as $tp=|\{e \in E^+ | B \cup H \models e\}|$, the number of true negatives as $tn=|\{e \in E^- | B \cup H \not\models e\}|$, and the balanced accuracy $ba = (tp/p + tn/n)/2$.

Our experimental method is as follows. For each game $g$, each task $g_t$, and each set of metarules $m$ in the S-, E-, D-, and $D^*$-reductions:
\begin{enumerate}
    \item Learn a program $p$ using all the training examples for $g_t$ using the metarules $m$ with a timeout of 10 minutes
    \item Measure the balanced accuracy of $p$ using the testing examples
\end{enumerate}

\noindent
If no program is found in 10 minutes then no program is returned and every testing example is deemed to have failed.



\subsubsection{Results}

Table \ref{fig:game-results} shows the balanced accuracies when learning with the different sets of metarules.
Again, we are not interested in the absolute accuracies only the relative differences when learning using different sets of metarules.
The $D$ set outperforms the $S$ and $E$ sets with a higher mean accuracy of 72\%, vs 66\% and 66\% respectively.
The $D^*$ set again outperforms them all with a mean accuracy of 73\%.
A McNemar's test on the $D$ and $D^*$ accuracies confirmed the significance at the $p < 0.01$ level.
Table \ref{fig:game-results} shows the corresponding learning times when varying the metarules.
Again, the $D$ set outperforms the $S$ and $E$ sets, and again the $D^*$ set outperforms them all.
However, a paired t-test on the $D$ and $D^*$ learning times confirmed the significance only at the $p < 0.08$ level, so the difference in learning times is insignificant.
Overall, the results from this experiment suggest that we can reject the null hypothesis in terms of predictive accuracies but not learning times.

\begin{table}[ht]
\centering
\normalsize
\begin{tabular}{c|c|c|c|c}
&
\textbf{S} &
\textbf{E} &
\textbf{D} &
\textbf{D$^*$}\\
\hline
Balanced accuracy (\%) & 66 & 66 & 72 & \textbf{73}\\
Learning time (seconds) & 316 & 316 & 327 & \textbf{296}
\end{tabular}
\caption{Experimental results on the IGGP data.}
\label{fig:game-results}
\end{table}

\section{Conclusions and further work}
\label{sec:conclusions}

As stated in Section \ref{sec:intro}, despite the widespread use of metarules, there is little work determining which metarules to use for a given learning task.
Instead, suitable metarules are assumed to be given as part of the background knowledge, or are used without any theoretical justification.
Deciding which metarules to use for a given learning task is a major open challenge \cite{crop:thesis,crop:minmeta} and is a trade-off between efficiency and expressivity: the hypothesis space grows given more metarules \cite{mugg:metabias,crop:minmeta}, so we wish to use fewer metarules, but if we use too few metarules then we lose expressivity.
To address this issue, Cropper and Muggleton \cite{crop:minmeta} used E-reduction on sets of metarules and showed that learning with E-reduced sets of metarules can lead to higher predictive accuracies and lower learning times compared to learning with non-E-reduced sets.
However, as we claimed in Section \ref{sec:intro}, E-reduction is not always the most appropriate form of reduction because it can remove metarules necessary to learn programs with the necessary specificity.

To support our claim, we have compared three forms of logical reduction: S-, E-, and D-reduction, where the latter is a new form of reduction based on SLD-derivations.
We have used the reduction algorithms to reduce finite sets of metarules.
Table \ref{tab:summary} summarises the results.
We have shown that many sets of metarules relevant to ILP do not have finite reductions (Theorem \ref{thm:irreduicble}).
These negative results have direct (negative) implications for MIL.
Specifically, our results mean that, in certain cases, a MIL system, such as Metagol or HEXMIL \cite{hexmil}, cannot be given a finite set of metarules from which it can learn any program, such as when learning arbitrary Datalog programs.
The results will also likely have implications for other forms of ILP which rely on metarules.

Our experiments compared learning the performance of Metagol when using the different reduced sets of metarules.
In general, using the D-reduced set outperforms both the S- and E-reduced sets in terms of predictive accuracy and learning time.
Our experimental results give strong evidence to our claim.
We also compared a $D^*$-reduced set, a subset of the D-reduced metarules, which, although derivationally incomplete, outperforms the other two sets in terms of predictive accuracies and learning times.


\subsection{Limitations and future work}

Theorem \ref{thm:irreduicble} shows that certain fragments of metarules do not have finite D-reductions.
However, our experimental results show that using D-reduced sets of metarules leads to higher predictive accuracies and lower learning times compared to the other forms of reduction.
Therefore, our work now opens up a new challenge of overcoming this negative theoretical result.
One idea is to explore whether special metarules, such as a currying metarule \cite{crop:metafunc}, could alleviate the issue.

In future work we would also like reduce more general fragments of logic, such as triadic logics, which would allow us to tackle a wider variety or problems, such as more of the games in the IGGP dataset.

We have compared the learning performance of Metagol when using different reduced sets of metarules.
However, we have not investigated whether these reductions are optimal.
For instance, when considering derivation reductions, it may, in some cases, be beneficial to re-add redundant metarules to the reduced sets to avoid having to derive them through SLD-resolution.
In future work, we would like to investigate identifying an optimal set of metarules for a given learning task, or preferably learning which metarules to use for a given learning task.

We have shown that although incomplete the $D^*$-reduced set of metarules outperforms the other reductions.
In future work we would like to explore other methods which sacrifice completeness for efficiency.


We have used the logical reduction techniques to remove redundant metarules.
It may also be beneficial to simultaneously reduce metarules and standard background knowledge.
The idea of purposely removing background predicates is similar to dimensionality reduction, widely used in other forms of machine learning \cite{skillicorn:understanding}, but which has been under researched in ILP \cite{furnkranz:dimensionality}.
Initial experiments indicate that this is possible \cite{crop:thesis,crop:minmeta}, and we aim to develop this idea in future work.



\begin{acknowledgements}
The authors thank Stephen Muggleton and Katsumi Inoue for discussions on this topic. We especially thank Rolf Morel for valuable feedback on the paper.
\end{acknowledgements}

\bibliographystyle{plain}
\bibliography{manuscript}

\begin{thebibliography}{10}

\bibitem{albarghouthi2017constraint}
Aws Albarghouthi, Paraschos Koutris, Mayur Naik, and Calvin Smith.
\newblock Constraint-based synthesis of {D}atalog programs.
\newblock In J.~Christopher Beck, editor, {\em Principles and Practice of
  Constraint Programming - 23rd International Conference, {CP} 2017, Melbourne,
  VIC, Australia, August 28 - September 1, 2017, Proceedings}, volume 10416 of
  {\em Lecture Notes in Computer Science}, pages 689--706. Springer, 2017.

\bibitem{bienvenu2007prime}
Meghyn Bienvenu.
\newblock Prime implicates and prime implicants in modal logic.
\newblock In {\em Proceedings of the Twenty-Second {AAAI} Conference on
  Artificial Intelligence, July 22-26, 2007, Vancouver, British Columbia,
  Canada}, pages 379--384. {AAAI} Press, 2007.

\bibitem{blumer:bound}
Anselm Blumer, Andrzej Ehrenfeucht, David Haussler, and Manfred~K. Warmuth.
\newblock Occam's razor.
\newblock {\em Inf. Process. Lett.}, 24(6):377--380, 1987.

\bibitem{bradley2007calculus}
Aaron~R. Bradley and Zohar Manna.
\newblock {\em The calculus of computation - decision procedures with
  applications to verification}.
\newblock Springer, 2007.

\bibitem{andreas}
A.~{Campero}, A.~{Pareja}, T.~{Klinger}, J.~{Tenenbaum}, and S.~{Riedel}.
\newblock {Logical Rule Induction and Theory Learning Using Neural Theorem
  Proving}.
\newblock {\em ArXiv e-prints}, September 2018.

\bibitem{church:problem}
Alonzo Church.
\newblock A note on the {E}ntscheidungsproblem.
\newblock {\em J. Symb. Log.}, 1(1):40--41, 1936.

\bibitem{cohen:grammarbias}
William~W. Cohen.
\newblock Grammatically biased learning: Learning logic programs using an
  explicit antecedent description language.
\newblock {\em Artif. Intell.}, 68(2):303--366, 1994.

\bibitem{crop:thesis}
Andrew Cropper.
\newblock {\em Efficiently learning efficient programs}.
\newblock PhD thesis, Imperial College London, {UK}, 2017.

\bibitem{iggp}
Andrew {Cropper}, Richard {Evans}, and Mark {Law}.
\newblock {Inductive general game playing}.
\newblock {\em arXiv e-prints}, page arXiv:1906.09627, Jun 2019.

\bibitem{crop:minmeta}
Andrew Cropper and Stephen~H. Muggleton.
\newblock Logical minimisation of meta-rules within meta-interpretive learning.
\newblock In Jesse Davis and Jan Ramon, editors, {\em Inductive Logic
  Programming - 24th International Conference, {ILP} 2014, Nancy, France,
  September 14-16, 2014, Revised Selected Papers}, volume 9046 of {\em Lecture
  Notes in Computer Science}, pages 62--75. Springer, 2014.

\bibitem{crop:metagolo}
Andrew Cropper and Stephen~H. Muggleton.
\newblock Learning efficient logical robot strategies involving composable
  objects.
\newblock In Qiang Yang and Michael Wooldridge, editors, {\em Proceedings of
  the Twenty-Fourth International Joint Conference on Artificial Intelligence,
  {IJCAI} 2015, Buenos Aires, Argentina, July 25-31, 2015}, pages 3423--3429.
  {AAAI} Press, 2015.

\bibitem{crop:metafunc}
Andrew Cropper and Stephen~H. Muggleton.
\newblock Learning higher-order logic programs through abstraction and
  invention.
\newblock In Subbarao Kambhampati, editor, {\em Proceedings of the Twenty-Fifth
  International Joint Conference on Artificial Intelligence, {IJCAI} 2016, New
  York, NY, USA, 9-15 July 2016}, pages 1418--1424. {IJCAI/AAAI} Press, 2016.

\bibitem{metagol}
Andrew Cropper and Stephen~H. Muggleton.
\newblock Metagol system.
\newblock https://github.com/metagol/metagol, 2016.

\bibitem{crop:metaopt}
Andrew Cropper and Stephen~H. Muggleton.
\newblock Learning efficient logic programs.
\newblock {\em Machine Learning}, 108(7):1063--1083, 2019.

\bibitem{crop:datacurate}
Andrew Cropper, Alireza Tamaddoni{-}Nezhad, and Stephen~H. Muggleton.
\newblock Meta-interpretive learning of data transformation programs.
\newblock In Katsumi Inoue, Hayato Ohwada, and Akihiro Yamamoto, editors, {\em
  Inductive Logic Programming - 25th International Conference, {ILP} 2015,
  Kyoto, Japan, August 20-22, 2015, Revised Selected Papers}, volume 9575 of
  {\em Lecture Notes in Computer Science}, pages 46--59. Springer, 2015.

\bibitem{crop:dreduce}
Andrew Cropper and Sophie Tourret.
\newblock Derivation reduction of metarules in meta-interpretive learning.
\newblock In Fabrizio Riguzzi, Elena Bellodi, and Riccardo Zese, editors, {\em
  Inductive Logic Programming - 28th International Conference, {ILP} 2018,
  Ferrara, Italy, September 2-4, 2018, Proceedings}, volume 11105 of {\em
  Lecture Notes in Computer Science}, pages 1--21. Springer, 2018.

\bibitem{dantsin:lp}
Evgeny Dantsin, Thomas Eiter, Georg Gottlob, and Andrei Voronkov.
\newblock Complexity and expressive power of logic programming.
\newblock {\em {ACM} Comput. Surv.}, 33(3):374--425, 2001.

\bibitem{echenim2015quantifierfree}
Mnacho Echenim, Nicolas Peltier, and Sophie Tourret.
\newblock Quantifier-free equational logic and prime implicate generation.
\newblock In Amy~P. Felty and Aart Middeldorp, editors, {\em Automated
  Deduction - {CADE-25} - 25th International Conference on Automated Deduction,
  Berlin, Germany, August 1-7, 2015, Proceedings}, volume 9195 of {\em Lecture
  Notes in Computer Science}, pages 311--325. Springer, 2015.

\bibitem{emde:metarules}
Werner Emde, Christopher Habel, and Claus{-}Rainer Rollinger.
\newblock The discovery of the equator or concept driven learning.
\newblock In Alan Bundy, editor, {\em Proceedings of the 8th International
  Joint Conference on Artificial Intelligence. Karlsruhe, FRG, August 1983},
  pages 455--458. William Kaufmann, 1983.

\bibitem{evans:dilp}
Richard Evans and Edward Grefenstette.
\newblock Learning explanatory rules from noisy data.
\newblock {\em J. Artif. Intell. Res.}, 61:1--64, 2018.

\bibitem{dialogs}
Pierre Flener.
\newblock Inductive logic program synthesis with {DIALOGS}.
\newblock In Stephen Muggleton, editor, {\em Inductive Logic Programming, 6th
  International Workshop, ILP-96, Stockholm, Sweden, August 26-28, 1996,
  Selected Papers}, volume 1314 of {\em Lecture Notes in Computer Science},
  pages 175--198. Springer, 1996.

\bibitem{Fonseca:ilp04}
Nuno~A. Fonseca, V{\'{\i}}tor~Santos Costa, Fernando M.~A. Silva, and Rui
  Camacho.
\newblock On avoiding redundancy in inductive logic programming.
\newblock In Rui Camacho, Ross~D. King, and Ashwin Srinivasan, editors, {\em
  Inductive Logic Programming, 14th International Conference, {ILP} 2004,
  Porto, Portugal, September 6-8, 2004, Proceedings}, volume 3194 of {\em
  Lecture Notes in Computer Science}, pages 132--146. Springer, 2004.

\bibitem{furnkranz:dimensionality}
Johannes Fürnkranz.
\newblock Dimensionality reduction in {ILP}: A call to arms.
\newblock In {\em Proceedings of the IJCAI-97 Workshop on Frontiers of
  Inductive Logic Programming}, pages 81--86, 1997.

\bibitem{subsumption-npcomplete}
M.~R. Garey and David~S. Johnson.
\newblock {\em Computers and Intractability: {A} Guide to the Theory of
  NP-Completeness}.
\newblock W. H. Freeman, 1979.

\bibitem{genesereth}
Michael~R. Genesereth, Nathaniel Love, and Barney Pell.
\newblock General game playing: Overview of the {AAAI} competition.
\newblock {\em {AI} Magazine}, 26(2):62--72, 2005.

\bibitem{gottlob1993removing}
Georg Gottlob and Christian~G. Ferm{\"{u}}ller.
\newblock Removing redundancy from a clause.
\newblock {\em Artif. Intell.}, 61(2):263--289, 1993.

\bibitem{gottlob:complexity}
Georg Gottlob, Nicola Leone, and Francesco Scarcello.
\newblock On the complexity of some inductive logic programming problems.
\newblock In Nada Lavrac and Saso Dzeroski, editors, {\em Inductive Logic
  Programming, 7th International Workshop, ILP-97, Prague, Czech Republic,
  September 17-20, 1997, Proceedings}, volume 1297 of {\em Lecture Notes in
  Computer Science}, pages 17--32. Springer, 1997.

\bibitem{DBLP:conf/ijcai/HemaspaandraS11}
Edith Hemaspaandra and Henning Schnoor.
\newblock Minimization for generalized boolean formulas.
\newblock In Toby Walsh, editor, {\em {IJCAI} 2011, Proceedings of the 22nd
  International Joint Conference on Artificial Intelligence, Barcelona,
  Catalonia, Spain, July 16-22, 2011}, pages 566--571. {IJCAI/AAAI}, 2011.

\bibitem{heule2015clause}
Marijn Heule, Matti J{\"{a}}rvisalo, Florian Lonsing, Martina Seidl, and Armin
  Biere.
\newblock Clause elimination for {SAT} and {QSAT}.
\newblock {\em J. Artif. Intell. Res.}, 53:127--168, 2015.

\bibitem{DBLP:conf/birthday/HillenbrandPWW13}
Thomas Hillenbrand, Ruzica Piskac, Uwe Waldmann, and Christoph Weidenbach.
\newblock From search to computation: Redundancy criteria and simplification at
  work.
\newblock In Andrei Voronkov and Christoph Weidenbach, editors, {\em
  Programming Logics - Essays in Memory of Harald Ganzinger}, volume 7797 of
  {\em Lecture Notes in Computer Science}, pages 169--193. Springer, 2013.

\bibitem{DBLP:journals/jacm/Joyner76}
William H.~Joyner Jr.
\newblock Resolution strategies as decision procedures.
\newblock {\em J. {ACM}}, 23(3):398--417, 1976.

\bibitem{hexmil}
Tobias Kaminski, Thomas Eiter, and Katsumi Inoue.
\newblock Exploiting answer set programming with external sources for
  meta-interpretive learning.
\newblock {\em {TPLP}}, 18(3-4):571--588, 2018.

\bibitem{mobal}
J{\"o}rg-Uwe Kietz and Stefan Wrobel.
\newblock Controlling the complexity of learning in logic through syntactic and
  task-oriented models.
\newblock In {\em Inductive logic programming}. Citeseer, 1992.

\bibitem{sld-resolution}
Robert~A. Kowalski.
\newblock Predicate logic as programming language.
\newblock In {\em {IFIP} Congress}, pages 569--574, 1974.

\bibitem{michalski:trains}
J.~Larson and Ryszard~S. Michalski.
\newblock Inductive inference of {VL} decision rules.
\newblock {\em {SIGART} Newsletter}, 63:38--44, 1977.

\bibitem{liberatore1}
Paolo Liberatore.
\newblock Redundancy in logic {I:} {CNF} propositional formulae.
\newblock {\em Artif. Intell.}, 163(2):203--232, 2005.

\bibitem{liberatore2}
Paolo Liberatore.
\newblock Redundancy in logic {II:} 2{CNF} and {H}orn propositional formulae.
\newblock {\em Artif. Intell.}, 172(2-3):265--299, 2008.

\bibitem{mugg:metabias}
Dianhuan Lin, Eyal Dechter, Kevin Ellis, Joshua~B. Tenenbaum, and Stephen
  Muggleton.
\newblock Bias reformulation for one-shot function induction.
\newblock In {\em {ECAI} 2014 - 21st European Conference on Artificial
  Intelligence, 18-22 August 2014, Prague, Czech Republic - Including
  Prestigious Applications of Intelligent Systems {(PAIS} 2014)}, pages
  525--530, 2014.

\bibitem{lloyd:book}
John~W. Lloyd.
\newblock {\em Foundations of Logic Programming, 2nd Edition}.
\newblock Springer, 1987.

\bibitem{lloyd:logiclearning}
J.W. Lloyd.
\newblock {\em Logic for Learning}.
\newblock Springer, Berlin, 2003.

\bibitem{horn:undecidable}
Jerzy Marcinkowski and Leszek Pacholski.
\newblock Undecidability of the {H}orn-clause implication problem.
\newblock In {\em 33rd Annual Symposium on Foundations of Computer Science,
  Pittsburgh, Pennsylvania, USA, 24-27 October 1992}, pages 354--362, 1992.

\bibitem{marquis2000consequence}
Pierre Marquis.
\newblock Consequence finding algorithms.
\newblock In {\em Handbook of {Defeasible} {Reasoning} and {Uncertainty}
  {Management} {Systems}}, pages 41--145. Springer, 2000.

\bibitem{DBLP:conf/mi/McCarthy95}
John McCarthy.
\newblock Making robots conscious of their mental states.
\newblock In {\em Machine Intelligence 15, Intelligent Agents [St. Catherine's
  College, Oxford, July 1995]}, pages 3--17, 1995.

\bibitem{crop:typed}
Rolf Morel, Andrew Cropper, and C.{-}H.~Luke Ong.
\newblock Typed meta-interpretive learning of logic programs.
\newblock In Francesco Calimeri, Nicola Leone, and Marco Manna, editors, {\em
  Logics in Artificial Intelligence - 16th European Conference, {JELIA} 2019,
  Rende, Italy, May 7-11, 2019, Proceedings}, volume 11468 of {\em Lecture
  Notes in Computer Science}, pages 198--213. Springer, 2019.

\bibitem{mugg:progol}
Stephen Muggleton.
\newblock Inverse entailment and {P}rogol.
\newblock {\em New Generation Comput.}, 13(3{\&}4):245--286, 1995.

\bibitem{mugg:golem}
Stephen Muggleton and Cao Feng.
\newblock Efficient induction of logic programs.
\newblock In {\em Algorithmic Learning Theory, First International Workshop,
  {ALT} '90, Tokyo, Japan, October 8-10, 1990, Proceedings}, pages 368--381,
  1990.

\bibitem{mlj:ilp20}
Stephen Muggleton, Luc~De Raedt, David Poole, Ivan Bratko, Peter~A. Flach,
  Katsumi Inoue, and Ashwin Srinivasan.
\newblock {ILP} turns 20 - biography and future challenges.
\newblock {\em Machine Learning}, 86(1):3--23, 2012.

\bibitem{mugg:metalearn}
Stephen~H. Muggleton, Dianhuan Lin, Niels Pahlavi, and Alireza
  Tamaddoni{-}Nezhad.
\newblock Meta-interpretive learning: application to grammatical inference.
\newblock {\em Machine Learning}, 94(1):25--49, 2014.

\bibitem{mugg:metagold}
Stephen~H. Muggleton, Dianhuan Lin, and Alireza Tamaddoni{-}Nezhad.
\newblock Meta-interpretive learning of higher-order dyadic {D}atalog:
  predicate invention revisited.
\newblock {\em Machine Learning}, 100(1):49--73, 2015.

\bibitem{nedellec1996declarative}
Claire N{\'e}dellec, C{\'e}line Rouveirol, Hilde Ad{\'e}, Francesco Bergadano,
  and Birgit Tausend.
\newblock Declarative bias in {ILP}.
\newblock {\em Advances in inductive logic programming}, 32:82--103, 1996.

\bibitem{ilp:book}
Shan-Hwei Nienhuys-Cheng and Ronald~de Wolf.
\newblock {\em Foundations of Inductive Logic Programming}.
\newblock Springer-Verlag New York, Inc., Secaucus, NJ, USA, 1997.

\bibitem{plotkin:thesis}
G.D. Plotkin.
\newblock {\em Automatic Methods of Inductive Inference}.
\newblock PhD thesis, Edinburgh University, August 1971.

\bibitem{raedt:decbias}
Luc~De Raedt.
\newblock Declarative modeling for machine learning and data mining.
\newblock In {\em Algorithmic Learning Theory - 23rd International Conference,
  {ALT} 2012, Lyon, France, October 29-31, 2012. Proceedings}, page~12, 2012.

\bibitem{raedt:clint}
Luc~De Raedt and Maurice Bruynooghe.
\newblock Interactive concept-learning and constructive induction by analogy.
\newblock {\em Machine Learning}, 8:107--150, 1992.

\bibitem{robinson:resolution}
John~Alan Robinson.
\newblock A machine-oriented logic based on the resolution principle.
\newblock {\em J. {ACM}}, 12(1):23--41, 1965.

\bibitem{schmid-schauss}
Manfred Schmidt{-}Schau{\ss}.
\newblock Implication of clauses is undecidable.
\newblock {\em Theor. Comput. Sci.}, 59:287--296, 1988.

\bibitem{shapiro:thesis}
E.Y. Shapiro.
\newblock {\em Algorithmic program debugging}.
\newblock MIT Press, 1983.

\bibitem{ALPS}
Xujie Si, Woosuk Lee, Richard Zhang, Aws Albarghouthi, Paraschos Koutris, and
  Mayur Naik.
\newblock Syntax-guided synthesis of {D}atalog programs.
\newblock In Gary~T. Leavens, Alessandro Garcia, and Corina~S. Pasareanu,
  editors, {\em Proceedings of the 2018 {ACM} Joint Meeting on European
  Software Engineering Conference and Symposium on the Foundations of Software
  Engineering, {ESEC/SIGSOFT} {FSE} 2018, Lake Buena Vista, FL, USA, November
  04-09, 2018}, pages 515--527. {ACM}, 2018.

\bibitem{skillicorn:understanding}
David Skillicorn.
\newblock {\em Understanding complex datasets: data mining with matrix
  decompositions}.
\newblock Chapman and Hall/CRC, 2007.

\bibitem{tarnlund:hornclause}
Sten{-}{\AA}ke T{\"{a}}rnlund.
\newblock Horn clause computability.
\newblock {\em {BIT}}, 17(2):215--226, 1977.

\bibitem{crop:sldres}
Sophie Tourret and Andrew Cropper.
\newblock {SLD}-resolution reduction of second-order {H}orn fragments.
\newblock In Francesco Calimeri, Nicola Leone, and Marco Manna, editors, {\em
  Logics in Artificial Intelligence - 16th European Conference, {JELIA} 2019,
  Rende, Italy, May 7-11, 2019, Proceedings}, volume 11468 of {\em Lecture
  Notes in Computer Science}, pages 259--276. Springer, 2019.

\bibitem{wang2014structure}
William~Yang Wang, Kathryn Mazaitis, and William~W. Cohen.
\newblock Structure learning via parameter learning.
\newblock In Jianzhong Li, Xiaoyang~Sean Wang, Minos~N. Garofalakis, Ian
  Soboroff, Torsten Suel, and Min Wang, editors, {\em Proceedings of the 23rd
  {ACM} International Conference on Conference on Information and Knowledge
  Management, {CIKM} 2014, Shanghai, China, November 3-7, 2014}, pages
  1199--1208. {ACM}, 2014.

\bibitem{DBLP:journals/aicom/WeidenbachW10}
Christoph Weidenbach and Patrick Wischnewski.
\newblock Subterm contextual rewriting.
\newblock {\em {AI} Commun.}, 23(2-3):97--109, 2010.

\end{thebibliography}

\newpage
\appendix

\section{Detailed Reduction Results}

\FloatBarrier
\subsection{Connected (\C{a}{m}) reductions}
\label{app:connected-reductions}

\begin{table}[h]
\footnotesize
\centering
    \begin{tabular}[t]{|c|c|c|}
    \hline
    S-reduction & E-reduction & D-reduction\\
    \hline
    \begin{tabular}[t]{l}
$P(A,B) \leftarrow Q(A,C)$\\
$P(A,B) \leftarrow Q(B,C)$\\
$P(A,B) \leftarrow Q(C,A)$\\
$P(A,B) \leftarrow Q(C,B)$\\
    \end{tabular}
    &
    \begin{tabular}[t]{l}
$P(A,B) \leftarrow Q(B,C)$
    \end{tabular}
    &
    \begin{tabular}[t]{l}
$P(A,A) \leftarrow Q(B,A)$\\
$P(A,B) \leftarrow Q(B,A)$\\
$P(A,B) \leftarrow Q(B,B)$\\
$P(A,B) \leftarrow Q(A,B),R(A,B)$\\
$P(A,B) \leftarrow Q(A,C),R(B,C)$\\
$P(A,B) \leftarrow Q(A,C),R(A,D),S(B,C),T(B,D),U(C,D)$
    \end{tabular}\\
    \hline
    \end{tabular}
\caption{Reductions of the connected fragment \C{\{2\}}{5}}
\end{table}

\begin{table}[h]
\footnotesize
\centering
    \begin{tabular}[t]{|c|c|c|}
    \hline
    S-reduction & E-reduction & D-reduction\\
    \hline
    \begin{tabular}[t]{l}
$P(A) \leftarrow Q(A)$\\
$P(A) \leftarrow Q(A,B)$\\
$P(A) \leftarrow Q(B,A)$\\
$P(A,B) \leftarrow Q(A)$\\
$P(A,B) \leftarrow Q(B)$\\
$P(A,B) \leftarrow Q(A,C)$\\
$P(A,B) \leftarrow Q(B,C)$\\
$P(A,B) \leftarrow Q(C,A)$\\
$P(A,B) \leftarrow Q(C,B)$\\
    \end{tabular}
    &
    \begin{tabular}[t]{l}
$P(A) \leftarrow Q(B,A)$\\
$P(A,B) \leftarrow Q(A)$
    \end{tabular}
    &
    \begin{tabular}[t]{l}
$P(A) \leftarrow Q(B,A)$\\
$P(A,A) \leftarrow Q(B,A)$\\
$P(A,B) \leftarrow Q(B)$\\
$P(A,B) \leftarrow Q(B,A)$\\
$P(A,B) \leftarrow Q(B,B)$\\
$P(A,B) \leftarrow Q(A,B),R(A,B)$\\
$P(A,B) \leftarrow Q(A,C),R(B,C)$\\
$P(A,B) \leftarrow Q(A,C),R(A,D),S(B,C),T(B,D),U(C,D)$
    \end{tabular}\\
    \hline
    \end{tabular}
\caption{Reductions of the connected fragment \C{\{1,2\}}{5}}
\end{table}

\FloatBarrier
\clearpage
\subsection{Datalog (\D{a}{m}) reductions}
\label{app:Datalog-reductions}
\begin{table}[h]
\footnotesize
\centering
    \begin{tabular}[t]{|c|c|c|}
    \hline
    S-reduction & E-reduction & D-reduction\\
    \hline
    \begin{tabular}[t]{l}
$P \leftarrow Q$\\
$P(A) \leftarrow Q(A)$\\
$P(A) \leftarrow Q(A,B)$\\
$P(A) \leftarrow Q(B,A)$\\
$P(A,A) \leftarrow Q(B,A)$\\
$P(A,B) \leftarrow Q(A,B)$\\
$P(A,B) \leftarrow Q(B,A)$\\
$P(A,B) \leftarrow Q(A),R(B)$\\
$P(A,B) \leftarrow Q(A),R(B,C)$\\
$P(A,B) \leftarrow Q(B),R(A,C)$\\
$P(A,B) \leftarrow Q(B,C),R(A,D)$
    \end{tabular}
    &
    \begin{tabular}[t]{l}
$P \leftarrow Q$\\
$P(A) \leftarrow Q(A,B)$\\
$P(A,B) \leftarrow Q(B,A)$\\
$P(A,B) \leftarrow Q(A),R(B)$\\
    \end{tabular}
    &
    \begin{tabular}[t]{l}
$P \leftarrow Q$\\
$P \leftarrow Q,R$\\
$P(A) \leftarrow Q(B,A)$\\
$P(A,A) \leftarrow Q(A)$\\
$P(A,A) \leftarrow Q(A,A)$\\
$P(A,B) \leftarrow Q(B,A)$\\
$P(A,B) \leftarrow Q,R(A,B)$\\
$P(A,B) \leftarrow Q(A,B),R(A,B)$\\
$P(A,B) \leftarrow Q(A,C),R(B,C)$\\
$P(A,B) \leftarrow Q(B,C),R(A,D)$\\
$P(A,B) \leftarrow Q(B,C),R(A,D),S(B,D),T(C,E)$\\
$P(A,B) \leftarrow Q(A,C),R(A,D),S(B,C),T(B,D),U(C,D)$\\
$P(A,B) \leftarrow Q(B,C),R(A,D),S(C,E),T(B,F),U(D,F)$\\
$P(A,B) \leftarrow Q(B,C),R(B,D),S(C,E),T(A,F),U(D,F)$\\
    \end{tabular}\\
    \hline
    \end{tabular}
\caption{Reductions of the Datalog fragment \D{\{0,1,2\}}{5}}
\vspace{2.5\baselineskip}
\footnotesize
\centering
    \begin{tabular}[t]{|c|c|c|}
    \hline
    S-reduction & E-reduction & D-reduction\\
    \hline
    \begin{tabular}[t]{l}
$P(A) \leftarrow Q(A)$\\
$P(A) \leftarrow Q(A,B)$\\
$P(A) \leftarrow Q(B,A)$\\
$P(A,A) \leftarrow Q(B,A)$\\
$P(A,B) \leftarrow Q(A),R(B)$\\
$P(A,B) \leftarrow Q(A),R(B,C)$\\
$P(A,B) \leftarrow Q(A,B)$\\
$P(A,B) \leftarrow Q(B),R(A,C)$\\
$P(A,B) \leftarrow Q(B,A)$\\
$P(A,B) \leftarrow Q(B,C),R(A,D)$\\
    \end{tabular}
    &
    \begin{tabular}[t]{l}
$P(A) \leftarrow Q(A,B)$\\
$P(A,B) \leftarrow Q(B,A)$\\
$P(A,B) \leftarrow Q(A),R(B)$\\
    \end{tabular}
    &
    \begin{tabular}[t]{l}
$P(A) \leftarrow Q(B,A)$\\
$P(A,A) \leftarrow Q(A)$\\
$P(A,A) \leftarrow Q(A,A)$\\
$P(A,B) \leftarrow Q(B,A)$\\
$P(A,B) \leftarrow Q(A,B),R(A,B)$\\
$P(A,B) \leftarrow Q(A,C),R(B,C)$\\
$P(A,B) \leftarrow Q(B,C),R(A,D)$\\
$P(A,B) \leftarrow Q(B,C),R(A,D),S(B,D),T(C,E)$\\
$P(A,B) \leftarrow Q(A,C),R(A,D),S(B,C),T(B,D),U(C,D)$\\
$P(A,B) \leftarrow Q(B,C),R(A,D),S(C,E),T(B,F),U(D,F)$\\
$P(A,B) \leftarrow Q(B,C),R(B,D),S(C,E),T(A,F),U(D,F)$
    \end{tabular}\\
    \hline
    \end{tabular}
\caption{Reductions of the Datalog fragment \D{\{1,2\}}{5}}
\vspace{2.5\baselineskip}
\footnotesize
\centering
    \begin{tabular}[t]{|c|c|c|}
    \hline
    S-reduction & E-reduction & D-reduction\\
    \hline
    \begin{tabular}[t]{l}
$P(A,A) \leftarrow Q(B,A)$\\
$P(A,B) \leftarrow Q(A,B)$\\
$P(A,B) \leftarrow Q(B,A)$\\
$P(A,B) \leftarrow Q(B,C),R(A,D)$\\
    \end{tabular}
    &
    \begin{tabular}[t]{l}
$P(A,A) \leftarrow Q(B,A)$\\
$P(A,B) \leftarrow Q(B,C),R(A,D)$
    \end{tabular}
    &
    \begin{tabular}[t]{l}
$P(A,A) \leftarrow Q(A,A)$\\
$P(A,A) \leftarrow Q(B,A)$\\
$P(A,B) \leftarrow Q(B,A)$\\
$P(A,B) \leftarrow Q(A,B),R(A,B)$\\
$P(A,B) \leftarrow Q(A,C),R(B,C)$\\
$P(A,B) \leftarrow Q(B,C),R(A,D)$\\
$P(A,B) \leftarrow Q(B,C),R(A,D),S(B,D),T(C,E)$\\
$P(A,B) \leftarrow Q(A,C),R(A,D),S(B,C),T(B,D),U(C,D)$\\
$P(A,B) \leftarrow Q(B,C),R(A,D),S(C,E),T(B,F),U(D,F)$\\
$P(A,B) \leftarrow Q(B,C),R(B,D),S(C,E),T(A,F),U(D,F)$
    \end{tabular}\\
    \hline
    \end{tabular}
\caption{Reductions of the Datalog fragment \D{\{2\}}{5}}

\end{table}

\FloatBarrier
\clearpage
\subsection{Singleton-free (\K{a}{m}) results}
\label{app:singleton-reductions}

\begin{table}[h]
\footnotesize
\centering
    \begin{tabular}[t]{|c|c|c|}
    \hline
    S-reduction & E-reduction & D-reduction\\
    \hline
    \begin{tabular}[t]{l}
$P(A,B) \leftarrow Q(A,B)$\\
$P(A,B) \leftarrow Q(B,A)$\\
$P(A,B) \leftarrow Q(B,C),R(A,D),S(A,D),T(B,C)$\\
    \end{tabular}
    &
    \begin{tabular}[t]{l}
$P(A,B) \leftarrow Q(B,A)$\\
$P(A,B) \leftarrow Q(A,A),R(B,B)$\\
$P(A,B) \leftarrow Q(A,C),R(B,C)$\\
    \end{tabular}
    &
    \begin{tabular}[t]{l}
$P(A,A) \leftarrow Q(A,A)$\\
$P(A,B) \leftarrow Q(B,A)$\\
$P(A,A) \leftarrow Q(A,B),R(B,B)$\\
$P(A,B) \leftarrow Q(A,A),R(B,B)$\\
$P(A,B) \leftarrow Q(A,B),R(A,B)$\\
$P(A,B) \leftarrow Q(A,C),R(B,C)$\\
$P(A,B) \leftarrow Q(A,C),R(A,D),S(B,C),T(B,D),U(C,D)$\\
    \end{tabular}\\
    \hline
    \end{tabular}
\caption{Reductions of the singleton-free fragment \K{\{2\}}{5}}
\end{table}

\begin{table}[h]
\footnotesize
\centering
    \begin{tabular}[t]{|c|c|c|}
    \hline
    S-reduction & E-reduction & D-reduction\\
    \hline
    \begin{tabular}[t]{l}
$P(A) \leftarrow Q(A)$\\
$P(A) \leftarrow Q(A,B),R(A,B)$\\
$P(A,B) \leftarrow Q(A,B)$\\
$P(A,B) \leftarrow Q(B,A)$\\
$P(A,B) \leftarrow Q(A),R(B)$\\
$P(A,B) \leftarrow Q(A),R(B,C),S(B,C)$\\
$P(A,B) \leftarrow Q(B),R(A,C),S(A,C)$\\
$P(A,B) \leftarrow Q(B,C),R(A,D),S(A,D),T(B,C)$\\
    \end{tabular}
    &
    \begin{tabular}[t]{l}
$P(A) \leftarrow Q(A,A)$\\
$P(A,B) \leftarrow Q(B,A)$\\
$P(A,B) \leftarrow Q(A),R(B)$\\
$P(A,B) \leftarrow Q(A,C),R(B,C)$\\
    \end{tabular}
    &
    \begin{tabular}[t]{l}
$P(A) \leftarrow Q(A,A)$\\
$P(A,A) \leftarrow Q(A)$\\
$P(A,B) \leftarrow Q(B,A)$\\
$P(A,A) \leftarrow Q(A,B),R(B,B)$\\
$P(A,B) \leftarrow Q(A,A),R(B,B)$\\
$P(A,B) \leftarrow Q(A,B),R(A,B)$\\
$P(A,B) \leftarrow Q(A,C),R(B,C)$\\
$P(A,B) \leftarrow Q(A,C),R(A,D),S(B,C),T(B,D),U(C,D)$
    \end{tabular}\\
    \hline
    \end{tabular}
\caption{Reductions of the singleton-free fragment \K{\{1,2\}}{5}}
\end{table}

\begin{table}[h]
\footnotesize
\centering
    \begin{tabular}[t]{|c|c|c|}
    \hline
    S-reduction & E-reduction & D-reduction\\
    \hline
    \begin{tabular}[t]{l}
$P \leftarrow Q$\\
$P(A) \leftarrow Q(A)$\\
$P(A) \leftarrow Q(A,B),R(A,B)$\\
$P(A,B) \leftarrow Q(A,B)$\\
$P(A,B) \leftarrow Q(B,A)$\\
$P(A,B) \leftarrow Q(A),R(B)$\\
$P(A,B) \leftarrow Q(A),R(B,C),S(B,C)$\\
$P(A,B) \leftarrow Q(B),R(A,C),S(A,C)$\\
$P(A,B) \leftarrow Q(B,C),R(A,D),S(A,D),T(B,C)$
    \end{tabular}
    &
    \begin{tabular}[t]{l}
$P \leftarrow Q$\\
$P(A) \leftarrow Q(A,A)$\\
$P(A,B) \leftarrow Q(B,A)$\\
$P(A,B) \leftarrow Q(A),R(B)$\\
$P(A,B) \leftarrow Q(A,C),R(B,C)$
    \end{tabular}
    &
    \begin{tabular}[t]{l}
$P \leftarrow Q$\\
$P \leftarrow Q,R$\\
$P(A) \leftarrow Q(A,A)$\\
$P(A,A) \leftarrow Q(A)$\\
$P(A,B) \leftarrow Q(B,A)$\\
$P(A,A) \leftarrow Q(A,B),R(B,B)$\\
$P(A,B) \leftarrow Q,R(A,B)$\\
$P(A,B) \leftarrow Q(A,A),R(B,B)$\\
$P(A,B) \leftarrow Q(A,B),R(A,B)$\\
$P(A,B) \leftarrow Q(A,C),R(B,C)$\\
$P(A,B) \leftarrow Q(A,C),R(A,D),S(B,C),T(B,D),U(C,D)$
    \end{tabular}\\
    \hline
    \end{tabular}
\caption{Reductions of the singleton-free fragment \K{\{0,1,2\}}{5}}
\end{table}

\FloatBarrier
\clearpage
\subsection{Duplicate-free (\U{a}{m}) results}
\label{app:duplicate-reductions}

\begin{table}[h]
\scriptsize
\centering
    \begin{tabular}[t]{|c|c|c|}
    \hline
    S-reduction & E-reduction & D-reduction\\
    \hline
    \begin{tabular}[t]{l}
$P(A,B) \leftarrow Q(A,B)$\\
$P(A,B) \leftarrow Q(B,A)$\\
$P(A,B) \leftarrow Q(B,C),R(A,D),S(A,D),T(B,C)$\\
    \end{tabular}
    &
    \begin{tabular}[t]{l}
      $P(A,B) \leftarrow Q(B,A)$\\
      $P(A,B) \leftarrow Q(B,C),R(A,D),$\\
      \qquad\qquad\quad$S(A,D),T(B,C)$\\
    \end{tabular}
    &
    \begin{tabular}[t]{l}
    $P(A,B) \leftarrow Q(B,A)$\\
    $P(A,B) \leftarrow Q(A,B),R(A,B)$\\
    $P(A,B) \leftarrow Q(A,C),R(B,C)$\\
    $P(A,B) \leftarrow Q(A,B),R(A,C),S(A,C)$\\
    $P(A,B) \leftarrow Q(A,B),R(A,C),S(C,D),T(C,D)$\\
    $P(A,B) \leftarrow Q(B,C),R(A,D),S(A,D),T(B,C)$\\
    $P(A,B) \leftarrow Q(B,C),R(A,D),S(B,C),T(B,D)$\\
    $P(A,B) \leftarrow Q(A,C),R(A,D),S(B,C),T(B,D),U(C,D)$\\
    $P(A,B) \leftarrow Q(B,C),R(A,D),S(B,D),T(C,E),U(C,E)$\\
    $P(A,B) \leftarrow Q(B,C),R(C,D),S(A,E),T(B,E),U(C,D)$\\
    \end{tabular}\\
    \hline
    \end{tabular}
\caption{Reductions of the fragment \U{\{2\}}{5}}
\end{table}

\begin{table}[h]
\scriptsize
\centering
    \begin{tabular}[t]{|c|c|c|}
    \hline
    S-reduction & E-reduction & D-reduction\\
    \hline
    \begin{tabular}[t]{l}
$P(A) \leftarrow Q(A)$\\
$P(A) \leftarrow Q(A,B),R(A,B)$\\
$P(A,B) \leftarrow Q(A,B)$\\
$P(A,B) \leftarrow Q(B,A)$\\
$P(A,B) \leftarrow Q(A),R(B)$\\
$P(A,B) \leftarrow Q(B),R(A,C),S(A,C)$\\
$P(A,B) \leftarrow Q(A),R(B,C),S(B,C)$\\
$P(A,B) \leftarrow Q(B,C),R(A,D),S(A,D),T(B,C)$\\
    \end{tabular}
    &
    \begin{tabular}[t]{l}
$P(A) \leftarrow Q(A,B),R(A,B)$\\
$P(A,B) \leftarrow Q(B,A)$\\
$P(A,B) \leftarrow Q(A),R(B)$\\
    \end{tabular}
    &
    \begin{tabular}[t]{l}
$P(A) \leftarrow Q(A)$\\
$P(A) \leftarrow Q(A),R(A)$\\
$P(A) \leftarrow Q(A,B),R(B)$\\
$P(A) \leftarrow Q(A,B),R(A,B)$\\
$P(A,B) \leftarrow Q(B,A)$\\
$P(A,B) \leftarrow Q(A),R(B)$\\
$P(A,B) \leftarrow Q(A),R(A,B)$\\
$P(A,B) \leftarrow Q(A,B),R(A,B)$\\
$P(A,B) \leftarrow Q(A,C),R(B,C)$\\
$P(A,B) \leftarrow Q(A,C),R(A,D),S(B,C),T(B,D),U(C,D)$\\
$P(A,B) \leftarrow Q(B,C),R(A,D),S(B,D),T(C,E),U(E)$\\
$P(A,B) \leftarrow Q(B,C),R(A,D),S(B,D),T(C,E),U(C,E)$\\
    \end{tabular}\\
    \hline
    \end{tabular}
\caption{Reductions of the fragment \U{\{1,2\}}{5}}
\end{table}

\begin{table}[h]
\scriptsize
\centering
    \begin{tabular}[t]{|c|c|c|}
    \hline
    S-reduction & E-reduction & D-reduction\\
    \hline
    \begin{tabular}[t]{l}
$P \leftarrow Q$\\
$P(A) \leftarrow Q(A)$\\
$P(A) \leftarrow Q(A,B),R(A,B)$\\
$P(A,B) \leftarrow Q(B,A)$\\
$P(A,B) \leftarrow Q(A,B)$\\
$P(A,B) \leftarrow Q(A),R(B)$\\
$P(A,B) \leftarrow Q(A),R(B,C),S(B,C)$\\
$P(A,B) \leftarrow Q(B),R(A,C),S(A,C)$\\
$P(A,B) \leftarrow Q(B,C),R(A,D),S(A,D),T(B,C)$\\
    \end{tabular}
    &
    \begin{tabular}[t]{l}
$P \leftarrow Q$\\
$P(A) \leftarrow Q(A,B),R(A,B)$\\
$P(A,B) \leftarrow Q(B,A)$\\
$P(A,B) \leftarrow Q(A),R(B)$\\
    \end{tabular}
    &
    \begin{tabular}[t]{l}
$P \leftarrow Q$\\
$P(A) \leftarrow Q(A)$\\
$P(A,B) \leftarrow Q(B,A)$\\
$P \leftarrow Q,R$\\
$P(A) \leftarrow Q,R(A)$\\
$P(A) \leftarrow Q(A),R(A)$\\
$P(A) \leftarrow Q(A,B),R(B)$\\
$P(A) \leftarrow Q(A,B),R(A,B)$\\
$P(A,B) \leftarrow Q,R(A,B)$\\
$P(A,B) \leftarrow Q(A),R(B)$\\
$P(A,B) \leftarrow Q(A),R(A,B)$\\
$P(A,B) \leftarrow Q(A,B),R(A,B)$\\
$P(A,B) \leftarrow Q(A,C),R(B,C)$\\
$P(A,B) \leftarrow Q(A,C),R(A,D),S(B,C),T(B,D),U(C,D)$\\
$P(A,B) \leftarrow Q(B,C),R(A,D),S(B,D),T(C,E),U(E)$\\
$P(A,B) \leftarrow Q(B,C),R(A,D),S(B,D),T(C,E),U(C,E)$
    \end{tabular}\\
    \hline
    \end{tabular}
\caption{Reductions of the fragment \U{\{0,1,2\}}{5}}
\end{table}

\end{document}